\documentclass[journal,twoside,web]{ieeecolor}
\usepackage{generic}
\usepackage{cite}
\usepackage{amsmath,amssymb,mathrsfs,amsfonts,bbm}
\usepackage{color}
\usepackage[ruled,vlined]{algorithm2e}
\usepackage{graphicx}
\usepackage{subcaption}
\usepackage{textcomp}
\usepackage{mathtools}
\usepackage[hidelinks]{hyperref}
\usepackage{dblfloatfix}
\usepackage{lipsum, cuted}

\newtheorem{problem}{Problem}
\newtheorem{definition}{Definition}

\newtheorem{theorem}{Theorem}

\newtheorem{remark}{Remark}

\newcommand{\Ad}{\mathrm{Ad}}
\newcommand{\ad}{\mathrm{ad}}

\newcommand{\T}{\mathrm{T}}
\newcommand{\cross}{^\wedge}

\newcommand{\black}[1]{{\color{black}#1}}

\begin{document}

\title{Full State Estimation of Continuum Robots from Tip Velocities: A Cosserat-Theoretic Boundary Observer}
\author{Tongjia Zheng, Qing Han, and Hai Lin
\thanks{*This work was supported by the National Science Foundation under Grant No. CNS-1830335, IIS-2007949.}
\thanks{Tongjia Zheng is with the Department of Electrical Engineering and the Department of Mathematics, University of Notre Dame, Notre Dame, IN 46556, USA. (tzheng1@nd.edu.)}
\thanks{Qing Han is with the Department of Mathematics, University of Notre Dame, Notre Dame, IN 46556, USA. (Qing.Han.7@nd.edu.)}
\thanks{Hai Lin is with the Department of Electrical Engineering, University of Notre Dame, Notre Dame, IN 46556, USA. (hlin1@nd.edu.)}
}

\maketitle

\thispagestyle{empty}
\pagestyle{empty}

\begin{abstract}
State estimation of robotic systems is essential to implementing feedback controllers, which usually provide better robustness to modeling uncertainties than open-loop controllers. However, state estimation of soft robots is very challenging because soft robots have theoretically infinite degrees of freedom while existing sensors only provide a limited number of discrete measurements. This work focuses on soft robotic manipulators, also known as continuum robots. We design an observer algorithm based on the well-known Cosserat rod theory, which models continuum robots by nonlinear partial differential equations (PDEs) evolving in geometric Lie groups. The observer can estimate all infinite-dimensional continuum robot states, including poses, strains, and velocities, by only sensing the tip velocity of the continuum robot, and hence it is called a ``boundary'' observer. More importantly, the estimation error dynamics is formally proven to be locally input-to-state stable. The key idea is to inject sequential tip velocity measurements into the observer in a way that dissipates the energy of the estimation errors through the boundary. \black{The distinct advantage of this PDE-based design is that it can be implemented using any existing numerical implementation for Cosserat rod models. All theoretical convergence guarantees will be preserved, regardless of the discretization method. We call this property ``one design for any discretization''.} Extensive numerical studies are included and suggest that the domain of attraction is large and the observer is robust to uncertainties of tip velocity measurements and model parameters.
\end{abstract}


\section{Introduction}
Soft robotics is a rapidly growing research area \cite{laschi2016soft}.
Thanks to their compliant properties, soft robots are safer when interacting with humans and are adaptable in constrained environments.
As a result, soft robots have found many applications, such as medical surgeries and underwater maneuvers \cite{castano2019model}.

Despite the empirical success, the theoretical study of soft robots has been considered a challenging problem.
In May 2023, IEEE Control Systems Magazine published a special issue that highlights control challenges for soft robotics \cite{sepulchre2023control}.
Over the past years, theoretical and experimental studies have suggested that feedback schemes are more robust to modeling uncertainties \cite{della2023model}.
Nevertheless, the perception of soft robots is also challenging because, theoretically, soft robots have infinite degrees of freedom. In contrast, existing sensing techniques can only provide a limited number of discrete measurements of the continuum states.
Moreover, some robot states (such as strains) are more difficult to measure than others (such as positions).
This work aims to develop algorithms to estimate these unknown infinite-dimensional states.
We focus on soft robotic manipulators, also known as continuum robots.

\black{Estimation problems are typically solved using model prediction, sensing, or their combination.
If we have access to a precise dynamic robot model, its exact initial state, and all inputs acting on it, then we can iterate the model to recover its state trajectory \cite{till2019real, mathew2022sorosim}.
This approach is known as \textit{model prediction}.
However, the compliant behavior of the robot makes it difficult to create a precise model, and there are unavoidable environmental disturbances.
As a result, relying solely on model prediction often increases deviations from the robot's actual motions over time.
Hence, a more reliable approach is to combine with sensing.}

When sensing techniques are used, the majority of existing work has focused on static or \textit{shape estimation}, which assumes that the robot is in a quasi-static state and aims to estimate the configuration of the entire continuum robot from discrete measurements of certain variables, such as position and orientation.
The common strategy is to fit a parametrized static (time-independent) spatial curve to the discrete measurements \cite{song2015electromagnetic, bezawada2022shape}.
The accuracy of this approach depends on the assumed curve model and the number of measurements.
Physically more plausible solutions are those that fit a mechanical equilibrium (represented by ordinary differential equations of the arc parameter) to the discrete measurements along the arc length, such as a static Kirchhoff rod \cite{anderson2017continuum} or a static Cosserat rod \cite{lilge2022continuum}.

Due to the quasi-static assumption, shape estimation methods have severe limitations, such as the restriction to slow-speed motions. 
Hence, the recent trend is to design dynamic estimators.  
Dynamic or \textit{state estimation} is an iterative process that uses a dynamic model and its inputs to predict new states and uses sequential sensor measurements to correct the prediction. 
The most widely adopted dynamic models of continuum robots include geometrical and continuum mechanics models \cite{armanini2023soft,della2023model}.
Geometrical models, especially piecewise constant curvature models, represent the continuum robot using a finite number of basis functions \cite{chirikjian1994hyper, della2020model}.
Continuum mechanics models, especially Cosserat rod models, benefit from a rigorous definition of the kinetic and potential energy of the system and usually take forms of nonlinear PDEs \cite{simo1988dynamics, macchelli2007port, rucker2011statics, renda2014dynamic} which are difficult to study.
Therefore, most of the existing work relies on finite-dimensional approximations, such as finite-dimensional Lagrangian \cite{grazioso2019geometrically, boyer2020dynamics} or port-Hamiltonian representations \cite{mattioni2020modelling}.
Based on discretized dynamic models, extended Kalman filters (EKFs) have been applied for continuum robots \cite{loo2019h, stewart2022state}.
Nevertheless, state estimation based on the original continuum mechanics PDEs has rarely been explored.
To our knowledge, the only existing work is our previous work \cite{zheng2022pde} in which a PDE-based EKF is reported.

In summary, existing work has revealed several limitations.
First, existing work has mainly relied on finite-dimensional approximations, which introduce additional modeling errors.
Second, existing work typically requires a large number of sensors to achieve good accuracy. 
\black{Third, no result regarding the convergence of estimation errors is available. We ask the following fundamental questions. 
\textit{Is it possible to recover all unknown infinite-dimensional states based on existing sensing techniques?
What is the minimum amount of necessary measurement?}
This work is devoted to these questions.}

In this work, we design a \textit{boundary observer} for continuum robots based on Cosserat rod theory \cite{simo1988dynamics, macchelli2007port, rucker2011statics, renda2014dynamic} and prove its (local) stability.
This algorithm is able to recover all infinite-dimensional robot states, including poses, strains, and velocities, using the PDE model, inputs, and only velocity measurements taken at the tip (which explains the name ``boundary'' observer).
The key idea is to inject sequential tip velocity measurements into the observer in a way that dissipates the energy of state estimation errors through the boundary.
It has three main advantages over the existing work.
\begin{enumerate}
    \item It only requires measuring the tip velocity.
    \item It can be implemented using any existing numerical method for Cosserat rod models.
    \item The state estimation error is proven to be locally input-to-state stable.
\end{enumerate}

\black{The second property means that we do not have to develop a new numerical implementation for this PDE-based observer algorithm.
Instead, any numerical method for Cosserat rod models can be used, such as those based on finite difference \cite{gazzola2018ForwardInverseProblems}, finite element \cite{grazioso2019geometrically}, strain parameterization \cite{renda2018discrete, boyer2020dynamics}, and shooting methods \cite{till2019real, boyer2022statics}, as well as those that may appear in the future.
The numerical implementation can be studied independently, and all theoretical convergence guarantees established in this work will be preserved.
We refer to this property as ``one design for any discretization''.}

Regarding the third property, to the best of our knowledge, this is the first work to construct a (locally) stable PDE-based observer for continuum robots.
To highlight its contribution, we point out that stability guarantees for nonlinear state estimation are difficult even for finite-dimensional systems.
Boundary estimation of PDEs is even harder because one needs to estimate infinite-dimensional states from only point measurements taken at the boundary.
The Cosserat rod PDE studied in this work is a semilinear hyperbolic system in the geometric Lie group $SE(3)$ \cite{murray2017mathematical}.
Although boundary estimation of certain general classes of hyperbolic PDEs has been studied \cite{vazquez2011backstepping,castillo2013boundary}, their assumptions, such as linearity or global Lipschitz continuity of the nonlinear terms, are not satisfied here.
The Lie group structure also poses additional difficulties to state estimation because the system states of the Cosserat rod PDE are defined in the local body frames while the effect of certain inputs, such as gravity, is defined in the global world frame.
In this regard, the (local) stability guarantee established in this work is also a novel contribution in the context of boundary estimation of PDE systems.
Extensive numerical studies are included, which suggest that the domain of attraction is large and that the observer is robust to uncertainties of tip measurements and model parameters.
The results suggest the promising role of PDE control theory in the theoretical study of continuum and soft robots.

The remainder of the paper is organized as follows.
In Section \ref{section:modeling}, we introduce the Cosserat rod model and the state estimation problem. 
In Section \ref{section:design}, we design a boundary observer and prove its stability.
In Section \ref{section:implementation}, we discuss its implementation in practice.
In Section \ref{section:simulation}, we conduct a series of numerical simulations to validate the performance and robustness of the boundary observer.
Section \ref{section:conclusion} summarizes the contribution and points out future directions.

\section{Modeling and Problem Formulation}
\label{section:modeling}

\subsection{Notations and Preliminaries} \label{section:notation}
We will use some Lie group notations from \cite{murray2017mathematical}.
Denote by $SO(3)$ the special orthogonal group (the group of rigid rotations) and by $so(3)$ its associated Lie algebra.
Denote by $SE(3)=SO(3)\times\mathbb{R}^3$ the special Euclidean group (the group of rigid motions) and by $se(3)$ its associated Lie algebra.
A hat $\wedge$ in the superscript of a vector $\eta$ defines a matrix $\eta\cross$ whose definition depends on the dimension of $\eta$.
Specifically, if $\eta\in\mathbb{R}^3$, then $\eta\cross\in so(3)$ is such that $\eta\cross\xi=\eta\times\xi$ for any $\xi\in\mathbb{R}^3$ where $\times$ is the cross product.
If $\eta=[w^T,v^T]^T\in\mathbb{R}^6$ with $w,v\in\mathbb{R}^3$, then $\eta\cross\in se(3)$ is defined by
\begin{align*}
    \eta\cross=
    \begin{bmatrix}
        w\cross & v \\
        0 & 0
    \end{bmatrix}\in\mathbb{R}^{4\times4}.
\end{align*}
Let the superscript $\vee$ be the inverse operator of $\wedge$, i.e., $(\eta\cross)^\vee=\eta$.
The adjoint operator $\ad$ of $\eta=[w^T~v^T]^T\in\mathbb{R}^6$ with $w,v\in\mathbb{R}^3$ is defined by
\begin{align*}
    \ad_\eta=
    \begin{bmatrix}
        w\cross & 0 \\
        v\cross & w\cross
    \end{bmatrix}\in\mathbb{R}^{6\times6}.
\end{align*}
By definition, $\ad$ is a linear operator and satisfies $\ad_\eta\xi=-\ad_\xi\eta$ for $\xi,\eta\in\mathbb{R}^6$.

\begin{figure}[h]
    \centering
    \includegraphics[width=0.8\columnwidth]{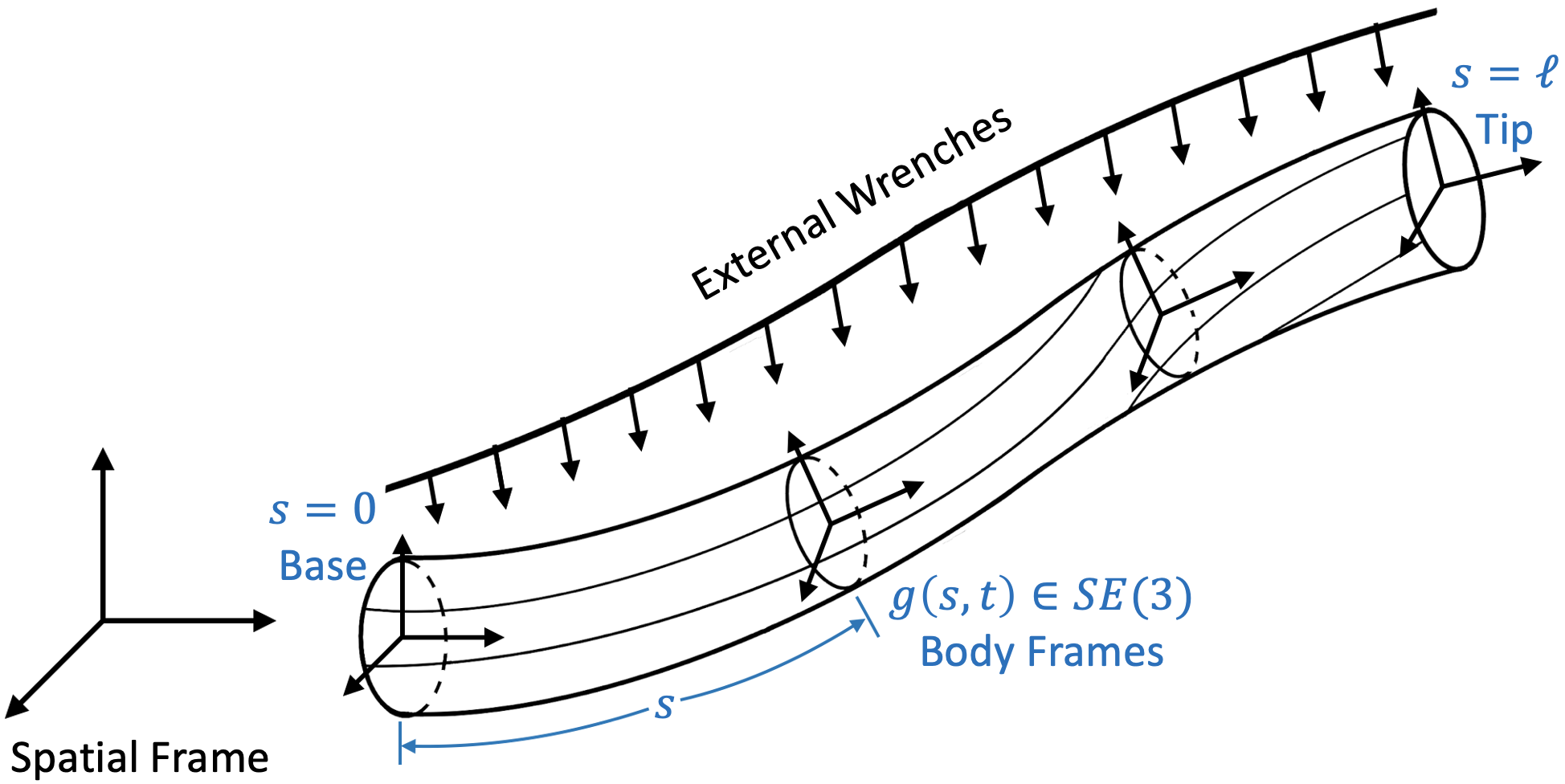}
    \caption{A Cosserat rod.}
    \label{fig:Cosserat rod}
\end{figure}

\subsection{Cosserat Rod Models for Continuum Robots}

Cosserat rod models are continuum mechanics models that describe the dynamic response of long and thin deformable rods undergoing external wrenches and have been widely used to model continuum robots \cite{macchelli2007port,rucker2011statics, renda2014dynamic, zheng2022task}.

\subsubsection{Configuration} 
A Cosserat rod is idealized as a continuous set of rigid cross-sections stacked along a centerline parameterized by the arc parameter $s\in[0,\ell]$; see Fig.~\ref{fig:Cosserat rod}.
Let $t\in[0,T]$ be time.
The pose of the entire rod is uniquely defined by a function $g(s,t)\in SE(3)$ given by
\begin{align*}
    g=
    \begin{bmatrix}
        R & p \\
        0 & 1
    \end{bmatrix},
\end{align*}
where $p(s,t)\in\mathbb{R}^3$ is the position vector of the centerline and $R(s,t)\in SO(3)$ is the rotation matrix of the cross-sections.
Note that there is a global frame while each cross-section also defines a local frame; see Fig.~\ref{fig:Cosserat rod}. 

\subsubsection{Kinematics} 
Let $w(s,t),v(s,t),u(s,t),q(s,t)\in\mathbb{R}^3$ be the fields of the angular velocity, linear velocity, angular strain, and linear strain, respectively, of the cross-sections in their local frames.
Let $\eta=[w^T~v^T]^T$ and $\xi=[u^T~q^T]^T$ be the fields of velocity twists and strain twists, respectively. 
The kinematics of the Cosserat rod is given by
\begin{align}
    \partial_tg=g\eta\cross, \label{eq:kinematics 1} \\
    \partial_sg=g\xi\cross, \label{eq:kinematics 2}
\end{align}
where $\partial_t:=\frac{\partial}{\partial t}$ and $\partial_s:=\frac{\partial}{\partial s}$ are partial derivatives.
The equality of mixed partial derivatives $\partial_{st}g=\partial_{ts}g$ yields the compatibility equation between the strain and the velocity
\begin{align}
    \partial_t\xi=\partial_s\eta+\ad_{\xi}\eta. \label{eq:compatibility}
\end{align}

\subsubsection{Dynamics}
Let $m(s,t),n(s,t),l(s,t),f(s,t)\in\mathbb{R}^3$ be the fields of the internal moment, internal force, external moment, and external force, respectively, of the cross-sections in their local frames. 
Let $\Phi=[m^T~n^T]^T$ and $\Psi=[l^T~f^T]^T$ be the fields of internal and external wrenches, respectively.
Applying Hamilton's principle in the context of Lie groups yields the following dynamics of the Cosserat rod in the form of nonlinear PDEs:
\begin{align} \label{eq:dynamics}
    J\partial_t\eta-\ad_\eta^TJ\eta=\partial_s\Phi-\ad_\xi^T\Phi+\Psi,
\end{align}
where $J(s)\in\mathbb{R}^{6\times6}$ is the cross-sectional inertia matrix.
We allow $J(s)$ to be a function of $s$ to account for, e.g., nonuniform cross-sectional areas.
Since this work is primarily interested in continuum robotic manipulators, the following boundary conditions at $s=0$ and $s=\ell$ are adopted
\begin{align}\label{eq:BC}
    & \eta(0,t)=\eta_-(t), \quad \Phi(\ell,t)=\Psi_+(t),
\end{align}
where $\eta_-(t)$ is the velocity of the station to which the continuum robot is attached and $\Psi_+(t)$ is the point wrench applied at the tip, such as a load at the tip.

\subsubsection{Inputs}
\black{The input wrenches of continuum robots can arise internally or externally, and their values can be specified in either the global frame, such as for gravity, or the local frames, such as for the forces from embedded actuators.
It will be convenient to define the operator $\T$ of $R$ by
\begin{align*}
    \T_R = 
    \begin{bmatrix}
        R & 0 \\
        0 & R
    \end{bmatrix} \in \mathbb{R}^{6\times6},
\end{align*}
which is used to transform a concatenation of moment and force from a coordinate system defined by $R$ to the global coordinate system.
For example,
\begin{align*}
    \begin{bmatrix}
        l_{\text{glb}} \\
        f_{\text{glb}}
    \end{bmatrix}
    =
    \begin{bmatrix}
        Rl_{\text{loc}} \\
        Rf_{\text{loc}}
    \end{bmatrix}
    = \T_R
    \begin{bmatrix}
        l_{\text{loc}} \\
        f_{\text{loc}}
    \end{bmatrix}.
\end{align*}
It is easy to verify that its inverse is given by $\T_R^T$.

\begin{remark}
Note that $\T_R$ differs from the Adjoint operator $\Ad_g$, which is used to transform velocity twists and wrenches between different frames \cite{murray2017mathematical}.
The transformation $\T_R$ only considers the change of orientation of the coordinate system and does not account for the change of origin.
It is the key transformation of forces and moments to establish the equivalence of the Newtonian formulation \cite{rucker2011statics} and the Lagrangian formulation \cite{renda2014dynamic} of Cosserat rods \cite{tummers2023cosserat}.
\end{remark}
}

We assume the wrench fields and point wrench at the tip take the following general forms:
\begin{align}
    \Phi(s,t) & =\phi(s,t)+\phi_{\text{loc}}(s,t), \label{eq:Phi} \\
    \Psi(s,t) & =\psi_{\text{loc}}(s,t)+\black{\T_{R(s,t)}^T}\psi_{\text{glb}}(s,t), \label{eq:Psi} \\
    \Psi_+(t) & =\psi_{\text{loc}}^+(t)+\black{\T_{R(\ell,t)}^T}\psi_{\text{glb}}^+(t). \label{eq:Psi tip}
\end{align}
In \eqref{eq:Phi}, $\phi_{\text{loc}}$ represents the wrench field applied internally with respect to the local frames, such as fluidic or tendon actuation.
$\phi$ is the wrench field due to elastic deformation and is assumed to satisfy the following linear constitutive law:
\begin{align} \label{eq:linear constitutive law}
    \phi=K(\xi-\xi_o),
\end{align}
where $K(s)\in\mathbb{R}^{6\times6}$ is the cross-sectional stiffness matrix and $\xi_o(s)$ is the reference strain field.
Again, $K(s)$ can be a function of $s$ due to nonuniform material properties.
In \eqref{eq:Psi}, $\psi_{\text{loc}}$ represents the distributed wrench applied externally whose value is specified in the local frames, and \black{$\T_{R}^T\psi_{\text{glb}}$} represents the distributed wrench applied externally whose value is specified by $\psi_{\text{glb}}$ in the global frame (like gravity and loads) and converted into the local frames through the coordinate transform $\T_{R}^T$.
It is important to distinguish between $\psi_{\text{loc}}$ and $\T_{R}^T\psi_{\text{glb}}$ for state estimation problems because the rotation $R$ is also an unknown robot state.
Similarly, in \eqref{eq:Psi tip}, $\psi_{\text{loc}}^+(t)$ represents the point wrench applied at the tip whose value is specified in the local frame and $\T_{R(\ell,t)}^T\psi_{\text{glb}}^+(t)$ represents the point wrench applied at the tip whose value is specified in the global world frame by $\psi_{\text{glb}}^+(t)$, such as a load at the tip.

\begin{remark}
A linear constitutive law with damping may be used to replace \eqref{eq:linear constitutive law} by
\begin{align}\label{eq:constitutive law with damping}
    \phi=K(\xi-\xi_o)+D\partial_t\xi,
\end{align}
where $D\in\mathbb{R}^{6\times6}$ is the damping matrix that models the viscoelastic property of the material.
Our numerical study shows that the boundary observer to be presented later performs well in this case.
However, from a theoretical perspective, \eqref{eq:constitutive law with damping} introduces mixed partial derivatives into the PDE system, making it much harder to study its stability.
More general nonlinear constitutive laws may also be used to replace \eqref{eq:linear constitutive law}, which will make the PDE system quasilinear instead of semilinear.
The extension to these general cases is left for future work.
\end{remark}

\subsection{Formulation of the Estimation Problem}

\black{To formulate a state estimation problem, it is important to look for a minimum representation of the system, i.e., the smallest set of system states and equations that uniquely determine the solution of the system.
For the Cosserat rod model, a minimum set of states is given by $\{g,\eta\}$ or $\{\xi,\eta\}$ because $g$ and $\xi$ are uniquely determined by each other through \eqref{eq:kinematics 2}.}
We find it convenient to work with $\{\xi,\eta\}$ because they satisfy the following semilinear hyperbolic system:
\begin{align}\label{eq:PDE system}
\left\{
\begin{aligned}
    & \partial_t\xi=\partial_s\eta+\ad_\xi\eta, \\
    & J\partial_t\eta=\partial_s(\phi+\phi_{\text{loc}})-\ad_{\xi}^T(\phi+\phi_{\text{loc}}) \\
    & \qquad\quad +\ad_{\eta}^TJ\eta+\psi_{\text{loc}}+\black{\T_R^T}\psi_{\text{glb}}, \\
    & \begin{aligned}
    & (\phi+\phi_{\text{loc}})(\ell,t)=\psi_{\text{loc}}^+(t)+\black{\T_{R(\ell,t)}^T}\psi_{\text{glb}}^+(t), \\
    & \eta(0,t)=\eta_-(t), \\
    & \phi(s,0)=\phi_0(s), \\
    & \eta(s,0)=\eta_0(s), 
    \end{aligned}
\end{aligned}
\right.
\end{align}
where $g$ and $\phi$ are determined by \eqref{eq:kinematics 2} and \eqref{eq:linear constitutive law} (or \eqref{eq:constitutive law with damping}) respectively at every $t$, and $\{\eta_0,\phi_0\}$ are the initial conditions.

Assume that the robot model \eqref{eq:PDE system}, its left-boundary condition $\eta_-$, all distributed inputs $\{\phi_{\text{loc}},\psi_{\text{loc}},\psi_{\text{glb}}\}$, and all right-boundary inputs $\{\psi_{\text{loc}}^+,\psi_{\text{glb}}^+\}$ are known.
\black{We will illustrate how to determine the model coefficients, boundary conditions, and inputs for a tendon-driven continuum robot in Section~\ref{section:implementation}.}
Note that the initial conditions do not have to be known.
Also, assume we can measure the tip velocity $\eta(\ell,t)$.
\black{This is called a boundary measurement and can be obtained by using an IMU and a motion capture system.
Since an IMU also measures orientation, we will assume that the tip orientation $R(\ell,t)$ is known for simplicity.}
(In fact, this assumption is needed only when $\psi_{\text{glb}}^+$ is nonzero).
As a result, we essentially assume that the right-boundary condition $\Psi_+$ is known.
We aim to estimate the continuum robot states $\{\xi,\eta\}$ based on the assumed information.
Once this is done, one can recover other robot states $\{g,\phi\}$ using \eqref{eq:kinematics 2} and \eqref{eq:linear constitutive law} (or \eqref{eq:constitutive law with damping}) at every $t$.
The state estimation problem is stated as follows.

\begin{problem}[Boundary Estimation]
Given the continuum robot model \eqref{eq:PDE system}, its boundary conditions $\{\eta_-(t),\Psi_+(t)\}$, distributed inputs $\{\phi_{\text{loc}}(s,t),\psi_{\text{loc}}(s,t),\psi_{\text{glb}}(s,t)\}$, and tip velocity measurements $\eta(\ell,t)$, design an algorithm to estimate the robot states $\{\xi(s,t),\eta(s,t)\}$.
\end{problem}

\begin{remark}
Boundary estimation of certain general classes of hyperbolic PDEs has been studied in the literature \cite{vazquez2011backstepping, castillo2013boundary}.
However, their results are not directly applicable to our case because our hyperbolic PDE \eqref{eq:PDE system} is semilinear, and the nonlinear terms are not globally Lipschitz continuous.
Another difficulty is due to the Lie group structure.
Note that the system states $\{\xi,\eta\}$ are defined in the local frames of the cross sections.
However, their dynamics depend on the global frame through \black{$\T_R^T\psi_{\text{glb}}$}, which includes gravity, and $R$ must be calculated from $\xi$ through spatial integration \eqref{eq:kinematics 2}.
This adds additional difficulty to the stability analysis.
\end{remark}

\section{Design and Stability of the Boundary Observer}
\label{section:design}

For an unknown state variable, say $\eta$, we use a hat $\wedge$ over the variable, i.e., $\hat{\eta}$, to denote its estimate.
We distinguish between $\hat{(\cdot)}$ and $(\cdot)\cross$ where the former is a state estimate and the latter is the hat operator defined in Section \ref{section:notation}.
Our estimation algorithm is called a \textit{boundary observer}.
The key idea is to inject the tip velocity measurement $\eta(\ell,t)$ into the observer in a way that dissipates the energy of the state estimation errors through the boundary.
The boundary observer is designed as:
\begin{align}\label{eq:boundary observer}
\hspace{-5.3pt}\left\{
\begin{aligned}
    & \partial_t\hat{\xi}=\partial_s\hat{\eta}+\ad_{\hat{\xi}}\hat{\eta}, \\
    & J\partial_t\hat{\eta}=\partial_s(\hat{\phi}+\phi_{\text{loc}})-\ad_{\hat{\xi}}^T(\hat{\phi}+\phi_{\text{loc}}) \\
    & \qquad\quad +\ad_{\hat{\eta}}^TJ\hat{\eta}+\psi_{\text{loc}}+\black{\T_{\hat{R}}^T\psi_{\text{glb}}}, \\
    & \begin{aligned}
    & (\hat{\phi}+\phi_{\text{loc}})(\ell,t)=\Psi_+(t)-\Gamma(\hat{\eta}-\eta)(\ell,t), \\
    & \hat{\eta}(0,t)=\eta_-(t), \\
    & \hat{\phi}(s,0)=\hat{\phi}_0(s), \\
    & \hat{\eta}(s,0)=\hat{\eta}_0(s),
    \end{aligned}
\end{aligned}
\right.
\end{align}
where the intermediate states $\{\hat{g},\hat{\phi}\}$ are computed using 
\begin{align}
    \partial_s\hat{g} & =\hat{g}\hat{\xi}\cross, \\
    \hat{\phi} & =K(\hat{\xi}-\xi_o), \label{eq:constitutive estimate} \\
    \text{or}~\hat{\phi} & =K(\hat{\xi}-\xi_o)+D\partial_t\hat{\xi}, \label{eq:constitutive estimate with damping}
\end{align}
at every $t$ according to \eqref{eq:kinematics 2} and \eqref{eq:linear constitutive law} (or \eqref{eq:constitutive law with damping}), $\Gamma\in\mathbb{R}^{6\times6}$ is a positive definite matrix representing the observer gain which can be used to adjust the performance of the observer, and $\{\hat{\phi}_0,\hat{\eta}_0\}$ are initial estimates, \black{which do not have to be the same as the actual initial states $\{\phi_0,\eta_0\}$}.

\begin{figure*}[h]
\begin{align}\label{eq:complete error PDE}
\begin{split}
    & \underbrace{\begin{bmatrix}
        K^{-1} & 0 \\
        0 & J
    \end{bmatrix}}_{\Lambda}\partial_t
    \begin{bmatrix}
        \tilde{\phi} \\
        \tilde{\eta}
    \end{bmatrix}+
    \underbrace{\begin{bmatrix}
        0 & -I \\
        -I & 0
    \end{bmatrix}}_{\Bar{A}}\partial_s
    \begin{bmatrix}
        \tilde{\phi} \\
        \tilde{\eta}
    \end{bmatrix}+
    \underbrace{\begin{bmatrix}
        0 & -\ad_{\xi_o} \\
        \ad_{\xi_o}^T & 0
    \end{bmatrix}}_{\Bar{B}}
    \begin{bmatrix}
        \tilde{\phi} \\
        \tilde{\eta}
    \end{bmatrix} \\
    = & 
    \underbrace{\begin{bmatrix}
        \ad_{K^{-1}\tilde{\phi}}\tilde{\eta} \\
        \ad_{\tilde{\eta}}^TJ\tilde{\eta}-\ad_{K^{-1}\tilde{\phi}}^T\tilde{\phi}
    \end{bmatrix}}_{\Bar{F}(\tilde{y})}
    +
    \underbrace{\begin{bmatrix}
        \ad_{K^{-1}\phi}\tilde{\eta}+\ad_{K^{-1}\tilde{\phi}}\eta \\
        -\ad_{K^{-1}\phi}^T\tilde{\phi}-\ad_{K^{-1}\tilde{\phi}}\phi+\ad_{\eta}^TJ\tilde{\eta}+\ad_{\tilde{\eta}}^TJ\eta
    \end{bmatrix}}_{\Bar{G}(y,\tilde{y})}+
    \underbrace{\begin{bmatrix}
        0 \\
        -\ad_{K^{-1}\tilde{\phi}}\phi_{\text{loc}}+
        \black{\T_{\tilde{R}}^T\psi_{\text{glb}}}
    \end{bmatrix}}_{\Bar{H}(d,\tilde{R},\tilde{y})}
\end{split}
\end{align}
\end{figure*}

This boundary observer has the classic observer structure in that it consists of a copy of the system plant plus injection of $(\hat{\eta}-\eta)(\ell,t)$, the estimation error of the tip velocity, through the boundary condition of $\hat{\phi}$.
The injected term is designed in such a way that it dissipates the energy of the estimation errors, which can be more clearly observed when we obtain the system of estimation errors \eqref{eq:compact error PDE}.
Such a boundary condition is called a dissipative boundary condition \cite{bastin2016stability}.
This boundary observer has three major advantages.
\begin{enumerate}
    \item It only requires measuring the velocity of the tip.
    \item It can be implemented using any existing numerical method for Cosserat rod models.
    In particular, the term $-\Gamma(\hat{\eta}-\eta)(\ell,t)$ can be numerically implemented as a virtual point wrench at the tip.
    \item The estimation error can be proven to be locally input-to-state stable (in the case of a linear constitutive law), which will be given in Theorem \ref{thm:stability} shortly.
\end{enumerate}

\begin{remark}
\black{The second property has a significant implication in practice. 
It means we do not need to develop a new numerical implementation for this observer.
Any numerical method for Cosserat rod models can be used \cite{gazzola2018ForwardInverseProblems, grazioso2019geometrically, renda2018discrete, boyer2020dynamics, till2019real, boyer2022statics}.
Given that our PDE-based design is eventually implemented based on certain discretization, one may wonder about its real advantage over a design based on discretized models.
Here is the answer.
First, PDE models are more compact and physically more interpretable.
If we discretize the PDE in the first place, it will be difficult to observe that a simple technique of boundary dissipation can produce a stable nonlinear observer.
Second, with a PDE-based design, the numerical implementation can be studied independently, and all established theoretical convergence guarantees will be preserved.
We refer to it as ``one design for any discretization''.}
\end{remark}

Now we prove the convergence of estimation errors.
We will always assume the linear constitutive law \eqref{eq:linear constitutive law}. 
Define
\begin{gather*}
    \tilde{\phi}=\hat{\phi}-\phi, \quad \tilde{\eta}=\hat{\eta}-\eta, \quad \tilde{\xi}=\hat{\xi}-\xi, \quad \tilde{R}=\hat{R}-R, \\
    y=
    \begin{bmatrix}
        \phi \\
        \eta
    \end{bmatrix}, \quad
    \hat{y}=
    \begin{bmatrix}
        \hat{\phi} \\
        \hat{\eta}
    \end{bmatrix}, \quad
    \tilde{y}=
    \begin{bmatrix}
        \tilde{\phi} \\
        \tilde{\eta}
    \end{bmatrix}, \quad
    d=
    \begin{bmatrix}
        \phi_{\text{loc}} \\
        \psi_{\text{glb}}
    \end{bmatrix}, \\
    y_0=
    \begin{bmatrix}
        \phi_0 \\
        \eta_0
    \end{bmatrix}, \quad
    \hat{y}_0=
    \begin{bmatrix}
        \hat{\phi}_0 \\
        \hat{\eta}_0
    \end{bmatrix}, \quad
    \tilde{y}_0=\hat{y}_0-y_0.
\end{gather*}
By \eqref{eq:linear constitutive law} and \eqref{eq:constitutive estimate},
\begin{align}\label{eq:tilde constitutive law}
    \tilde{\phi}=K\tilde{\xi}.
\end{align}

Subtracting \eqref{eq:boundary observer} from \eqref{eq:PDE system}, and by the linearity of $\ad$ and \eqref{eq:tilde constitutive law}, we obtain that $\tilde{\phi}$ satisfies
\begin{align*}
    K^{-1}\partial_t\tilde{\phi} & =\partial_t\tilde{\xi} \\
    & =\partial_s\tilde{\eta}+\ad_{\hat{\xi}}\hat{\eta}-\ad_\xi\eta \\
    & =\partial_s\tilde{\eta}+\ad_{\hat{\xi}}\hat{\eta}-\ad_{\hat{\xi}}\eta-\ad_\xi\tilde{\eta} \\
    & \quad+\ad_\xi\tilde{\eta}+\ad_{\hat{\xi}}\eta-\ad_\xi\eta \\
    & =\partial_s\tilde{\eta}+\ad_{\tilde{\xi}}\tilde{\eta}+\ad_\xi\tilde{\eta}+\ad_{\tilde{\xi}}\eta \\
    & =\partial_s\tilde{\eta}+\ad_{K^{-1}\tilde{\phi}}\tilde{\eta}+\ad_{(K^{-1}\phi+\xi_o)}\tilde{\eta}+\ad_{\tilde{\xi}}\eta,
\end{align*}
and by a similar derivation, $\tilde{\eta}$ satisfies
\begin{align*}
    J\partial_t\tilde{\eta} & =\partial_s\tilde{\phi}-\ad_{\tilde{\xi}}^T\tilde{\phi}-\ad_\xi^T\tilde{\phi}-\ad_{\tilde{\xi}}^T(\phi+\phi_{\text{loc}}) \\
    & \quad +\ad_{\tilde{\eta}}^TJ\tilde{\eta}+\ad_\eta^TJ\tilde{\eta}+\ad_{\tilde{\eta}}^TJ\eta+\black{\T_{\hat{R}}^T\psi_{\text{glb}}-\T_R^T\psi_{\text{glb}}} \\
    & =\partial_s\tilde{\phi}-\ad_{K^{-1}\tilde{\phi}}^T\tilde{\phi}+\ad_{\tilde{\eta}}^TJ\tilde{\eta}-\ad_{(K^{-1}\phi+\xi_o)}^T\tilde{\phi} \\
    & \quad -\ad_{K^{-1}\tilde{\phi}}^T(\phi+\phi_{\text{loc}})+\ad_\eta^TJ\tilde{\eta}+\ad_{\tilde{\eta}}^TJ\eta+\black{\T_{\tilde{R}}^T\psi_{\text{glb}}},
\end{align*}
with boundary conditions
\begin{align*}
    \tilde{\eta}(0,t)=0, \quad \tilde{\phi}(\ell,t)=-\Gamma\tilde{\eta}(\ell,t).
\end{align*}

The complete system of estimation errors can be written as a semilinear hyperbolic system given in \eqref{eq:complete error PDE}.
Note that the estimation errors do not depend on the boundary conditions $\{\eta_-,\Psi_+\}$ and the externally applied wrench $\psi_{\text{loc}}$ because they are completely compensated by the observer.
By left-multiplying \eqref{eq:complete error PDE} with $\Lambda^{-1}(s)$ (defined in \eqref{eq:complete error PDE}), we rewrite it in the following compact form:
\begin{align}\label{eq:compact error PDE}
\hspace{-1.8pt}\left\{
\begin{aligned}
    & \partial_t\tilde{y}+A\partial_s\tilde{y}+B\tilde{y}=F(\tilde{y})+G(y,\tilde{y})+H(d,\tilde{R},\tilde{y}), \\
    & \begin{aligned}
    & \tilde{\phi}(\ell,t)=-\Gamma\tilde{\eta}(\ell,t), \\
    & \tilde{\eta}(0,t)=0, \\
    & \tilde{y}(s,0)=\tilde{y}_0(s),
    \end{aligned}
\end{aligned}
\right.
\end{align}
where $A=\Lambda^{-1}\Bar{A}$, $B=\Lambda^{-1}\Bar{B}$, $F=\Lambda^{-1}\Bar{F}$, $G=\Lambda^{-1}\Bar{G}$, and $H=\Lambda^{-1}\Bar{H}$.
We observe that the boundary condition of $\tilde{\phi}$ behaves like a damping term that dissipates energy from the system \eqref{eq:compact error PDE} \cite{bastin2016stability}.
This term is the key to ensuring stability.

In the following theorem, input-to-state stability is used to study the stability of \eqref{eq:compact error PDE}.
\black{An introduction to this notion is included in the Appendix.}
Roughly speaking, a system is input-to-state stable if its solution is bounded by a positive function of external inputs and converges asymptotically in the absence of inputs.
In our case, we will treat $y$ and $d$ as external inputs/disturbances and establish that $\|\tilde{y}(\cdot,t)\|_{H^1}$ locally converges to a neighborhood bounded by a positive function of $\|y(\cdot,t)\|_{H^1}$ (the actual robot states) and $\|d(\cdot,t)\|_{H}$ (the inputs including gravity)\footnote{For a function $f(s)$, its $H^1$ norm is defined by $\|f\|_{H^1}:=\big(\int f^2+f'^2ds\big)^{1/2}$.
For a function $f(s,t)$, we define $\|f(\cdot,t)\|_{H}:=\big(\int f^2+(\partial_sf)^2+(\partial_tf)^2ds\big)^{1/2}$.
Roughly speaking, these norms include not only the function itself but also its partial derivatives.}.
The well-posedness of the PDE systems in this work has been studied in \cite{li2001semi, rodriguez2022networks}.
In the following theorem, we assume that its solution uniquely exists in the functional space $C^0([0,T];H^1(0,\ell))$, which is consistent with the results in \cite{rodriguez2022networks}.

\begin{theorem}\label{thm:stability}
Consider \eqref{eq:compact error PDE}.
If $\Gamma$ is positive definite, then the estimation error $\|\tilde{y}(\cdot,t)\|_{H^1}$ is locally input-to-state stable in the sense that there exist constants $k_0,k_1,k_2,b,\lambda,\kappa_1,\kappa_2>0$ such that for all \black{$\|\tilde{y}(\cdot,0)\|_{H^1}<k_0$}, $\|y(\cdot,t)\|_{H^1}<k_1$, $\|d(\cdot,t)\|_{H}<k_2$, and $t\geq0$, the following holds:
\begin{align}\label{eq:ISS}
    \|\tilde{y}(\cdot,t)\|_{H^1} & \leq b\|\tilde{y}(\cdot,0)\|_{H^1}e^{-\lambda t}+\kappa_1\sup_{0\leq\tau\leq t}\|y(\cdot,\tau)\|_{H^1} \nonumber \\
    & \quad +\kappa_2\sup_{0\leq\tau\leq t}\|d(\cdot,\tau)\|_{H}.
\end{align}
\end{theorem}

\begin{proof}
The proof is included in the Appendix.
\end{proof}

On the right-hand side of \eqref{eq:ISS}, the first term is due to the initial estimation error and decays exponentially.
The last two terms are proportional to certain norms of the actual states and inputs and, therefore, are always bounded in practice.
Although the theoretical convergence is local, numerical studies suggest that the domain of attraction is quite large.


\black{
\section{Implementation}
\label{section:implementation}

In this section, we illustrate how to implement our boundary observer using a tendon-driven continuum robot in Fig.~\ref{fig:continuum robot} as an example.
We assume the robot has a spring steel backbone and 20 equally spaced disks.
The robot is subject to gravity, the actuation of two tendons, and a load at the tip. 
The base is fixed.
First, we illustrate how to determine the model parameters $\{J(s),K(s)\}$, boundary conditions $\{\eta_-(t),\Psi_+(t)\}$, and distributed inputs $\{\phi_{\text{loc}}(s,t),\psi_{\text{loc}}(s,t),\psi_{\text{glb}}(s,t)\}$. 
Then, we comment on the numerical aspects.

\textit{Model Coefficients.} 
Denote the backbone radius by $r(s)$, density by $\rho(s)$, Young's modulus by $E(s)$, and Shear modulus by $G(s)$.
Note that the density usually needs to be adjusted to include the weights of the disks.
Denote the angular and linear inertia (stiffness) matrices by $J_1$ and $J_2$ ($K_1$ and $K_2$).
Assuming the $x$-axis aligns with the longitudinal direction of the robot, then \cite{rucker2011statics}
\begin{align*}
    J_1 & =\mathrm{diag}(2,1,1)\rho\pi r^4/4, \quad J_2  =I_{3\times3}\rho\pi r^2 , \\
    K_1 & =\mathrm{diag}(2G,E,E)\pi r^4/4, \quad K_2  =\mathrm{diag}(E,G,G)\pi r^2, \\
    J & =\mathrm{diag}(J_1,J_2),\quad K =\mathrm{diag}(K_1,K_2).
\end{align*}

\textit{Boundary Conditions.}
The fixed base implies $\eta_-(t)=0$.
The loaded tip means that $\Psi_+(t)$ consists of the gravity of the load converted into the local frame.

\textit{Distributed Inputs.}
The robot is subject only to gravity and tendon actuation.
Hence, $\psi_\text{loc}(s,t)=0$ and $\psi_\text{glb}(s,t)$ consists of only the gravity of the robot.
To determine $\phi_{\text{loc}}(s,t)$, we need an actuator model that calculates the generated wrench field from the actuator reading, such as fluidic pressure or tendon tension, which can be found in \cite{renda2020geometric}.
The general actuator model depends on the current state of the robot.
However, for tendon-driven continuum robots, this dependence is negligible (see Section III in \cite{boyer2022statics}).
In this case, the actuator model is given as follows.
Let $D_i(s)$ be the position of the intersection point of the tendon $i$ with the $s$-cross-section of the rod in the $s$-cross-sectional frame. 
Let $\tau_i(t)$ be the tendon force, which always takes a negative value.
Then
\begin{align}\label{eq:actuator model}
    \phi_{\text{loc}}(s,t)=\sum_{i=1}^2
    \begin{bmatrix}
        D_i(s)\cross T_i(s) \\
        T_i(s)
    \end{bmatrix}\frac{\tau_i(t)}{\|T_i(s)\|}
\end{align}
where $T_i=q_o+u_o\cross D_i+D_i'$ is the tangent of the tendon, and $q_o$ and $u_o$ are the reference strains.
Therefore, once the routing of the tendons is known, $\phi_{\text{loc}}(s,t)$ is uniquely determined by the tendon forces $\tau_i(t)$.

\textit{Measurements.}
The tip velocity measurement $\eta(\ell,t)$ includes linear and angular velocities.
The angular velocity can be obtained by installing an IMU at the tip.
The linear velocity can be obtained using a motion capture system.

\textit{Numerical Implementation.}
Since our observer is essentially a Cosserat rod with a virtual tip wrench, it can be implemented using any numerical method for Cosserat rod models, such as those based on finite difference \cite{gazzola2018ForwardInverseProblems}, finite element \cite{grazioso2019geometrically}, strain parameterization \cite{renda2018discrete, boyer2020dynamics}, and shooting methods \cite{till2019real, boyer2022statics}.
We only need to define an extra tip wrench in these numerical methods.
Some of these methods, such as shooting methods \cite{till2019real}, have been shown to be real-time.
Hence, it is also promising to implement our observer algorithms in real time.

\begin{figure}[h]
    \centering
    \includegraphics[width=0.6\columnwidth]{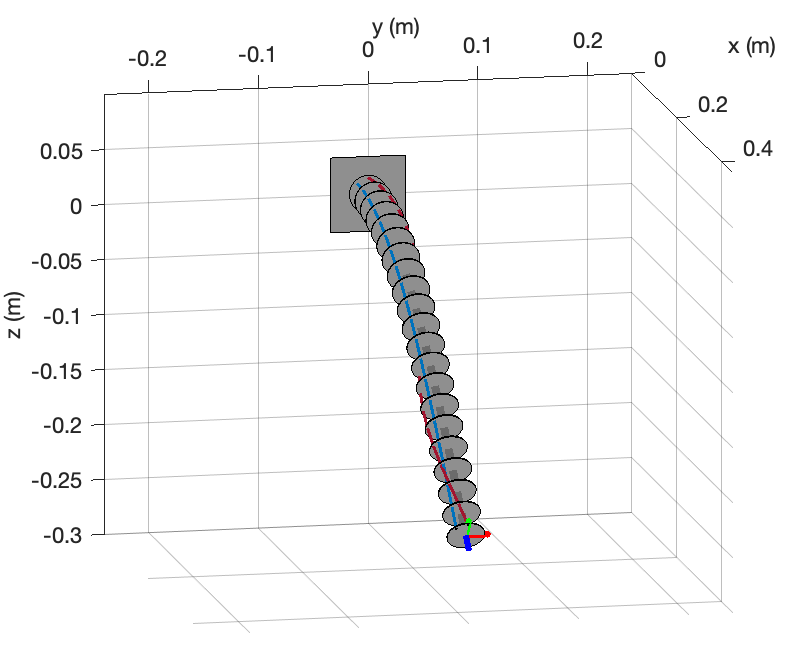}
    \caption{A continuum robot subject to gravity, actuation of two tendons, and a tip load of $1~\mathrm{N}$ (not plotted).}
    \label{fig:continuum robot}
\end{figure}

\begin{figure}[h]
    \centering
    \includegraphics[width=0.7\columnwidth]{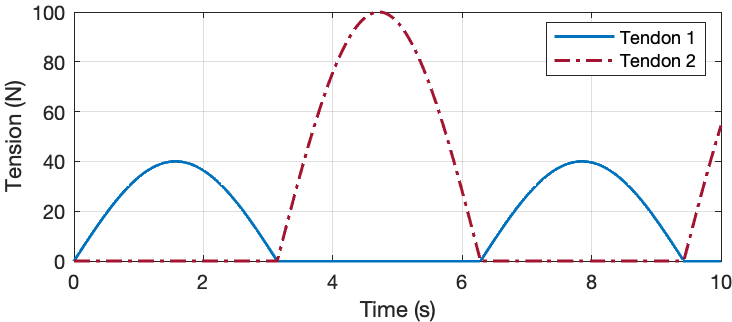}
    \caption{Tension of the tendons.}
    \label{fig:tendon tension}
\end{figure}

\begin{figure}[htb!]
    \centering
    \begin{subfigure}[b]{0.325\columnwidth}
        \centering
        \includegraphics[width=\textwidth]{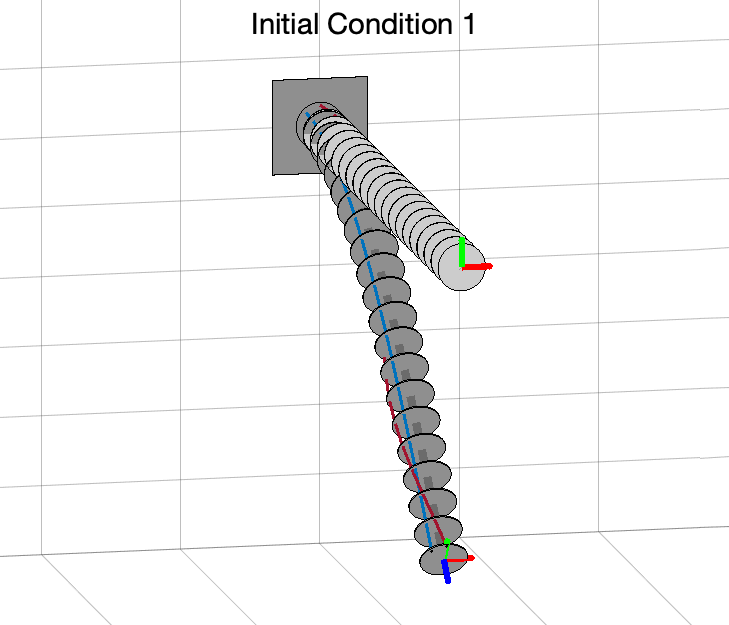}
    \end{subfigure}
    \begin{subfigure}[b]{0.325\columnwidth}
        \centering
        \includegraphics[width=\textwidth]{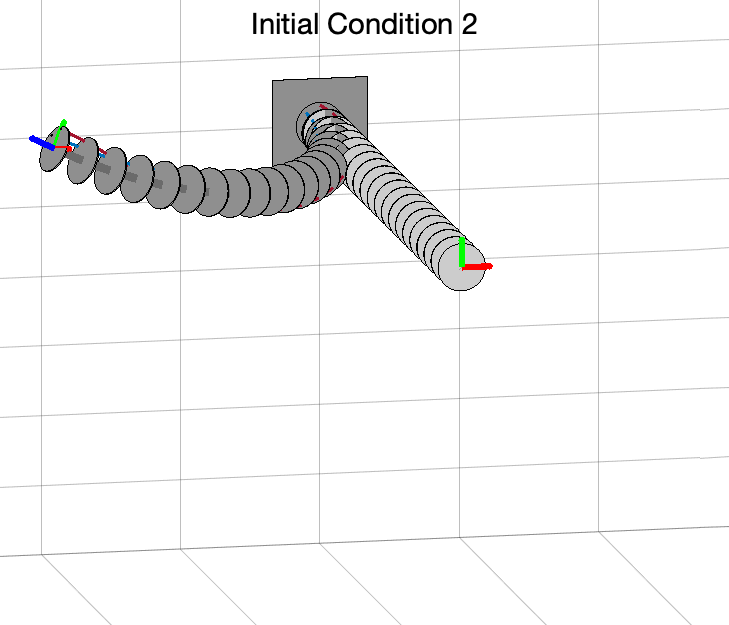}
    \end{subfigure}
    \begin{subfigure}[b]{0.325\columnwidth}
        \centering
        \includegraphics[width=\textwidth]{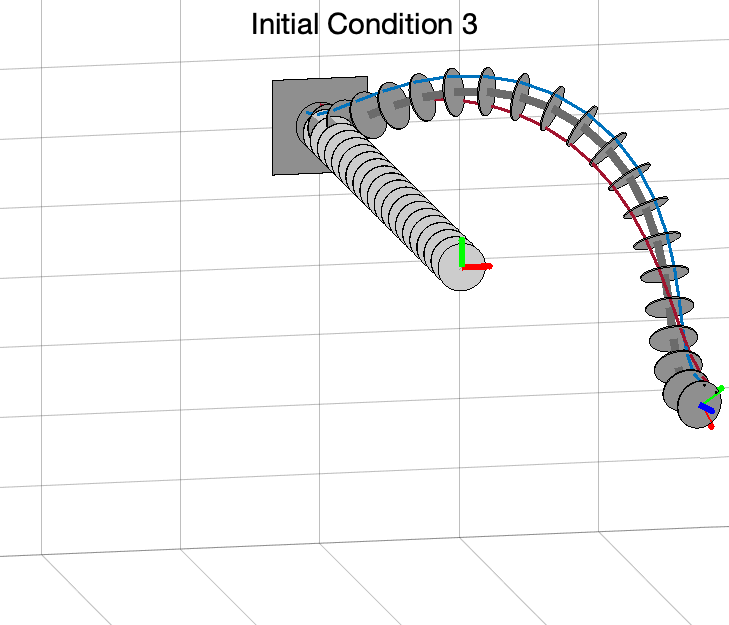}
    \end{subfigure}
    \caption{Three initial conditions of the actual states (dark gray) and the observer (light gray).}
    \label{fig:initial condition}
\end{figure}

\textit{Observer Gain.}
The observer gain $\Gamma$ is the only free parameter in our algorithm.
The estimation errors are convergent as long as $\Gamma$ is positive definite.
Theoretically, a larger $\Gamma$ results in faster convergence.
However, in numerical implementation, choosing a large $\Gamma$ means injecting a large wrench and may require smaller discretization steps for numerical stability.
Thus, selecting $\Gamma$ for real-time state estimation necessarily involves a compromise between the discretization in space and time, the computation speed, and the convergence speed of estimation errors.
This compromise also depends on the specific numerical method used and is currently under study.}

\section{Simulation study}
\label{section:simulation}

\black{In this section, we conduct extensive simulation studies on the performance and robustness of the presented observer algorithm using SoRoSim \cite{mathew2022sorosim}.
SoRoSim is a MATLAB simulator that implements a discretization scheme based on strain parameterization \cite{boyer2020dynamics}.
However, any other discretization scheme can be used \cite{gazzola2018ForwardInverseProblems, grazioso2019geometrically, renda2018discrete, till2019real, boyer2022statics}.

\begin{table}[h]
\caption{Robot parameters.}
\small
\setlength{\tabcolsep}{4pt}
\renewcommand{\arraystretch}{1.2}
    \centering
    \begin{tabular}{rl}
        \hline
        Length & $0.5~\mathrm{m}$ \\
        Radius & $1~\mathrm{mm}$ \\
        Density (including the disks) & $1.6\times10^4~\mathrm{kg/m^3}$ \\
        Young's modulus & $207~\mathrm{Gpa}$ \\
        Shear modulus & $79.6~\mathrm{Gpa}$ \\
        \hline
    \end{tabular}
    \label{tab:robot parameter}
\end{table}

\begin{figure*}[h]
    \centering
    \begin{subfigure}[b]{0.2\textwidth}
        \centering
        \includegraphics[width=\textwidth]{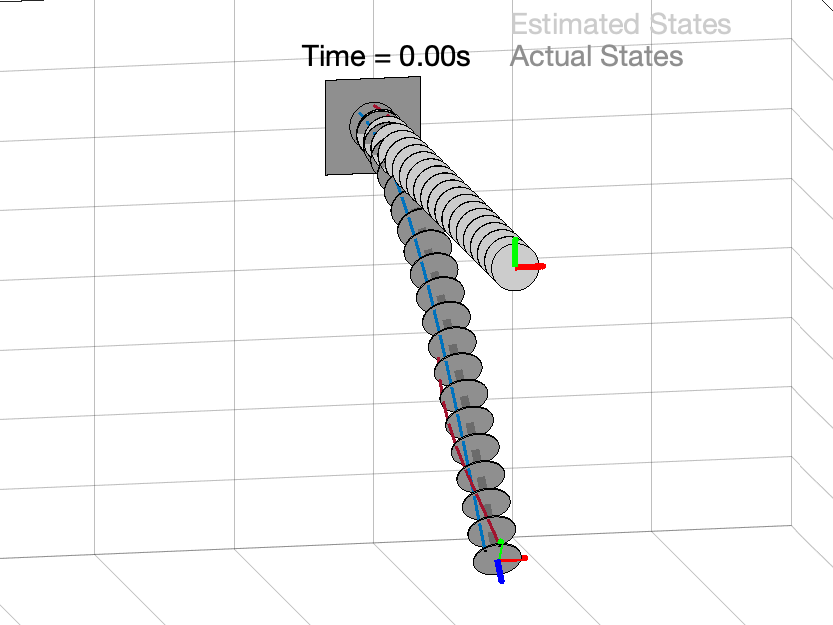}
    \end{subfigure}
    \begin{subfigure}[b]{0.2\textwidth}
        \centering
        \includegraphics[width=\textwidth]{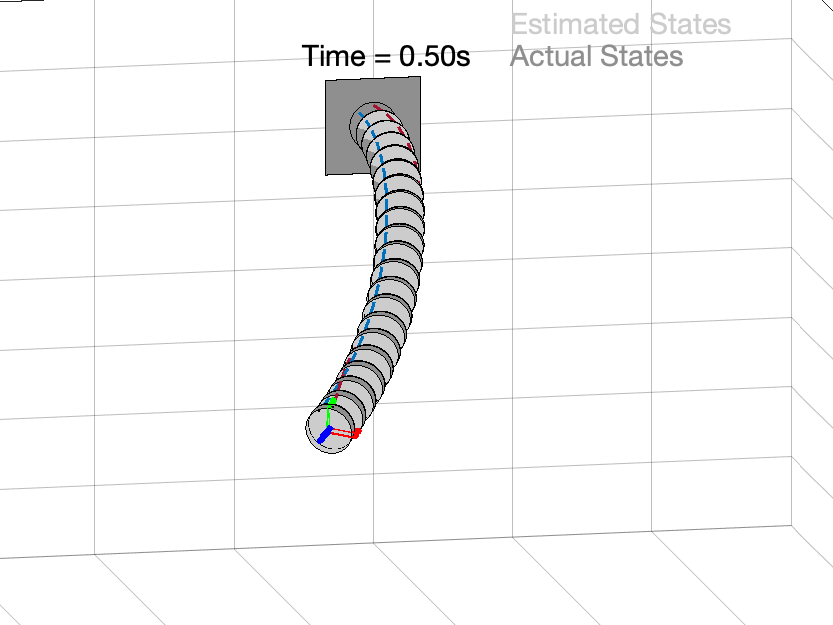}
    \end{subfigure}
    \begin{subfigure}[b]{0.2\textwidth}
        \centering
        \includegraphics[width=\textwidth]{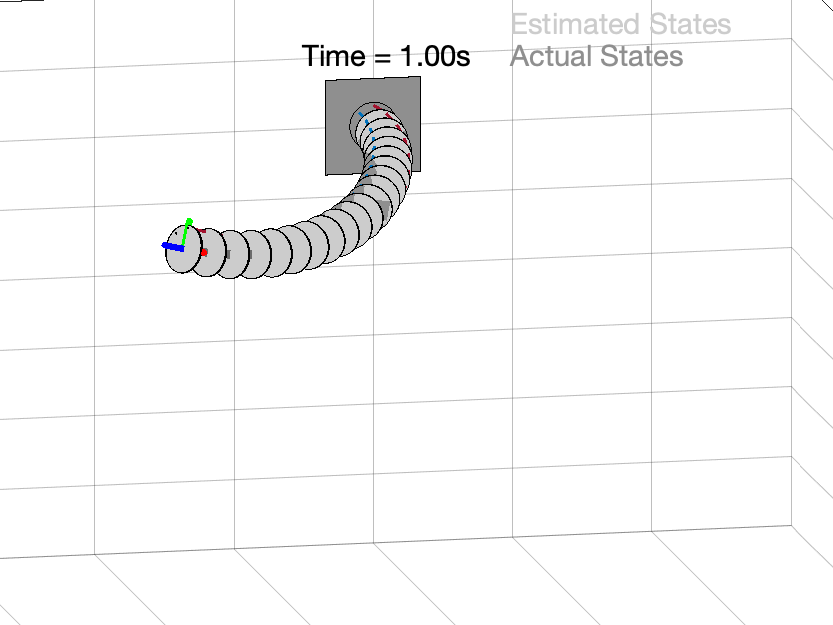}
    \end{subfigure}
    \begin{subfigure}[b]{0.2\textwidth}
        \centering
        \includegraphics[width=\textwidth]{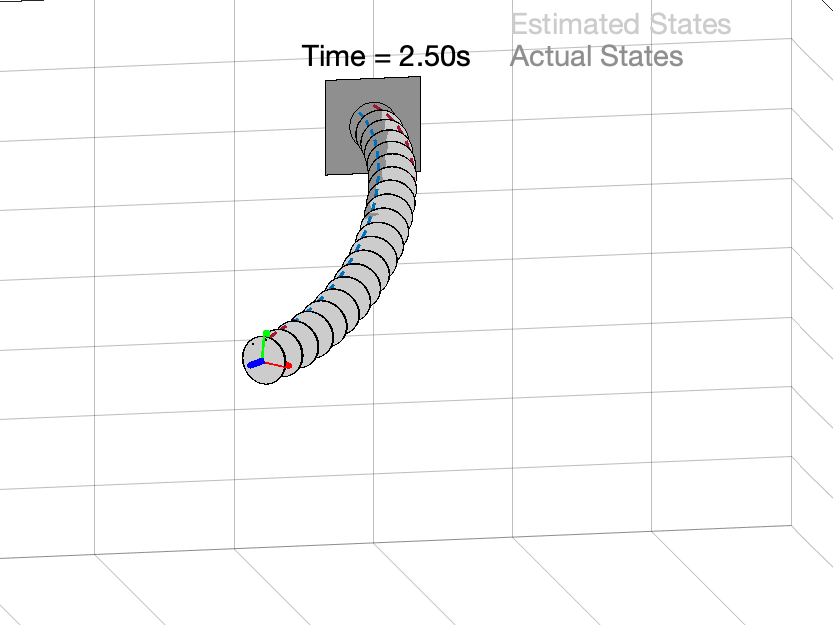}
    \end{subfigure}
    
    \begin{subfigure}[b]{0.2\textwidth}
        \centering
        \includegraphics[width=\textwidth]{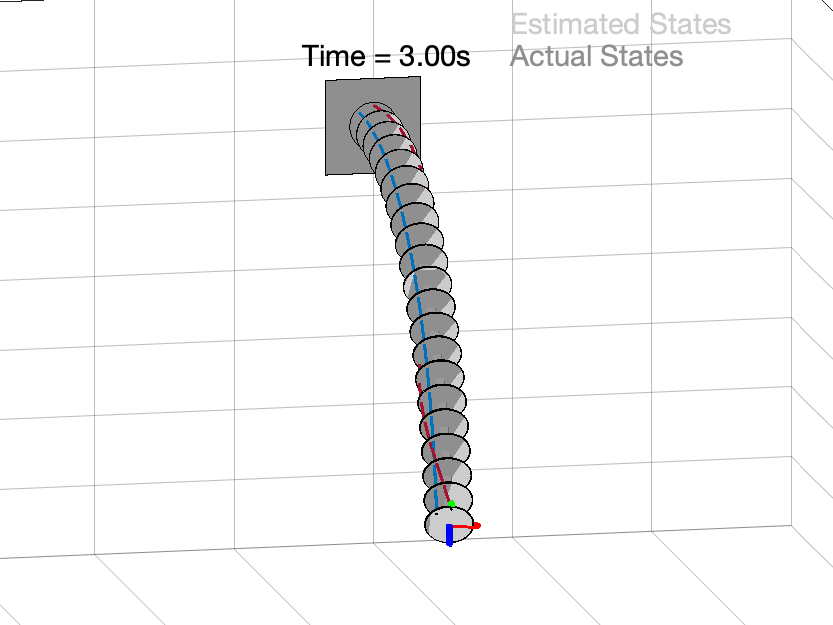}
    \end{subfigure}
    \begin{subfigure}[b]{0.2\textwidth}
        \centering
        \includegraphics[width=\textwidth]{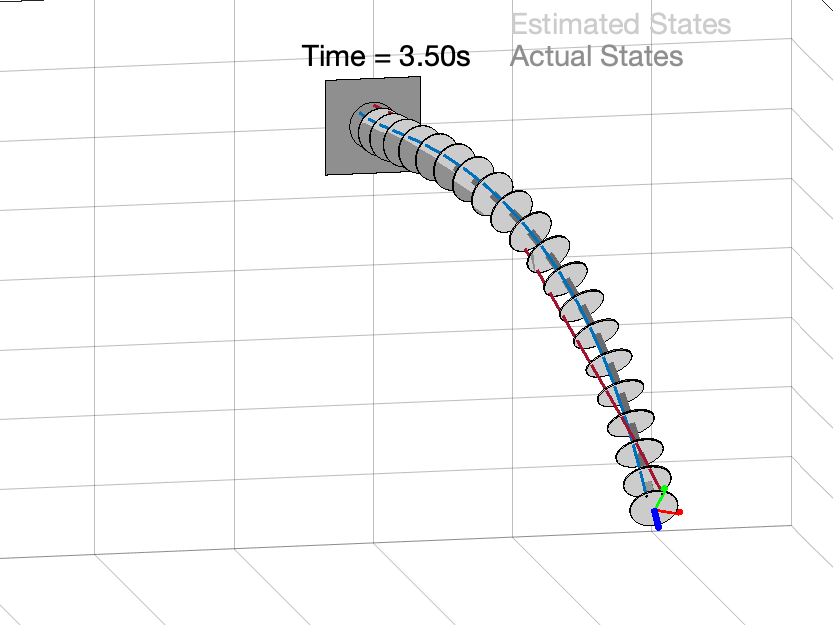}
    \end{subfigure}
    \begin{subfigure}[b]{0.2\textwidth}
        \centering
        \includegraphics[width=\textwidth]{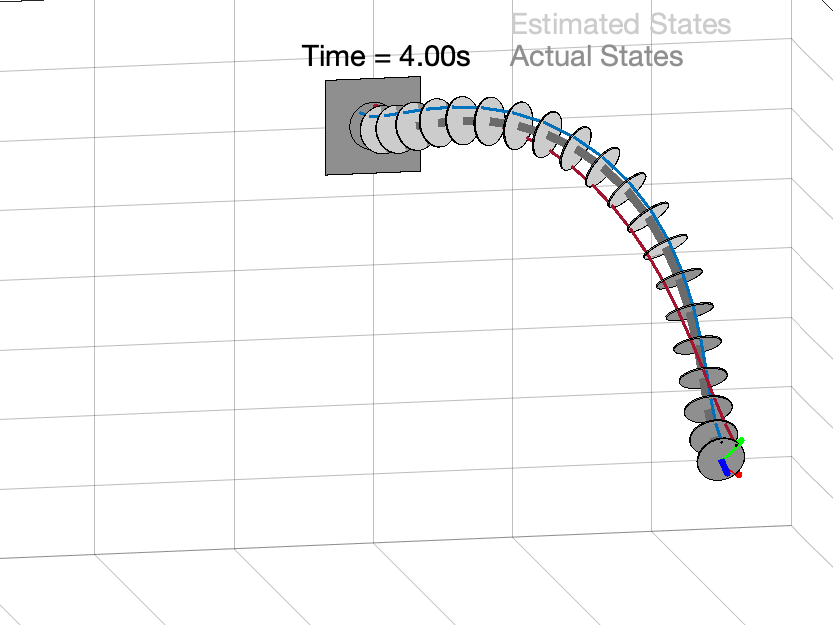}
    \end{subfigure}
    \begin{subfigure}[b]{0.2\textwidth}
        \centering
        \includegraphics[width=\textwidth]{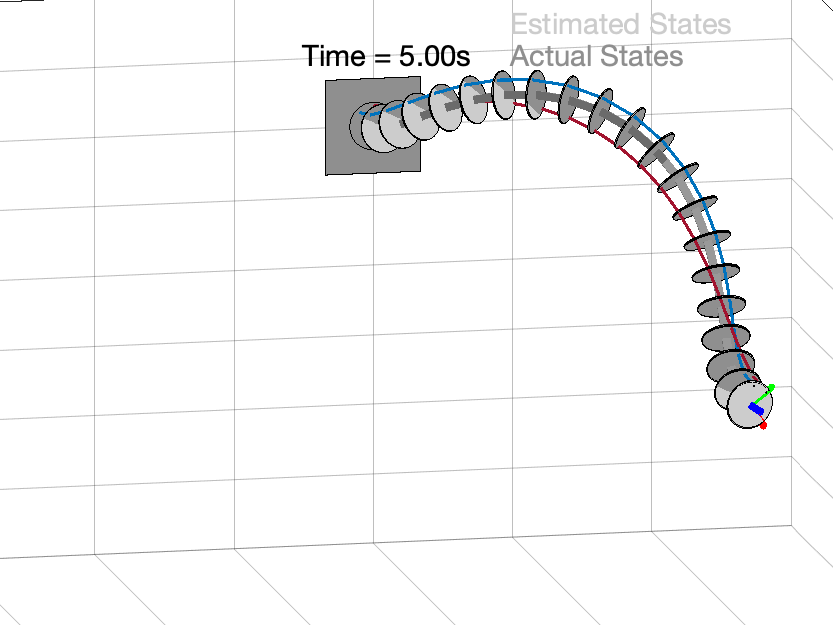}
    \end{subfigure}
    \caption{Comparison of the estimated configuration (light gray) and actual configuration (dark gray) for initial condition 1. After 0.5 seconds, the estimated configuration converged to the actual configuration and thus they overlapped.}
    \label{fig:convergence comparison}
\end{figure*}

The robot parameters are given in Table~\ref{tab:robot parameter}, adopted from \cite{rucker2011statics, till2019real}.
After including the weights of the disks, the density was about twice the density of spring steel.
Tendon 1 is parallel, and its relative position is given by $D_1(s)\equiv[0~-0.01~0.01]$.
Tendon 2 is helical, and its relative position is given by $D_2(s)\equiv[0~0.15\sin(4\pi s)~0.15\cos(4\pi s)]$.
The two tendons were pulled alternately to generate complicated 3D motions.
Their tensions were specified according to
\begin{align*}
    & \tau_1(t)=-[40\sin(t)]_+~\mathrm{N}, \\
    & \tau_2(t)=[100\sin(t)]_-~\mathrm{N},
\end{align*}
respectively, where $f_+$ and $f_-$ represent the positive and negative parts of the function $f$.

In each simulation study later, SoRoSim was run twice.
In the first run, it was used as a robot simulator to generate the ``true'' states (pose, strain, velocity).
The physical parameters and states were saved as the ground truth for comparison.
In the second run, it was used as a numerical method to implement our observer.
We used the same physical parameters but sometimes added perturbations.
We always started with a ``wrong'' initial estimate to demonstrate the convergence of our observer.
We used the saved tip velocity from the first run as the ``measurement'' and injected it into the observer in the second run.
If the computed solution of the second run converges to the saved solution of the first run, it means that the observer can recover the unknown states of the ``robot''.}

\black{To investigate the domain of attraction, the actual states began from three different initial configurations; see Fig.~\ref{fig:initial condition}.
Our observer's initial estimate was always set to a straight configuration, resulting in a significant initial estimation error.
The observer gain was $\Gamma=0.01I_{6\times6}$.

\textit{Result.}
A series of snapshots of the convergence process for initial condition 1 is plotted in Fig.~\ref{fig:convergence comparison}.
The robot first exhibited a simple bending motion due to the parallel tendon and then a complex twisting motion due to the helical tendon.
We observed that with accurate models and measurements, the estimates converged to the actual states in 0.5 seconds and exhibited close tracking of the actual states.
We also plotted the ground truth trajectories and estimated trajectories of various state variables, including the position (Fig.~\ref{fig:position tracking}), Euler angles (Fig.~\ref{fig:rotation tracking}), linear velocity (Fig.~\ref{fig:linear velocity tracking}), and angular velocity (Fig.~\ref{fig:angular velocity tracking}) at the tip, and the angular strains (Fig.~\ref{fig:angular strain tracking}) at the midpoint location.
(The linear strains of a spring backbone are negligible and hence were not plotted).
We observed that all estimates converged to the ground truth, including velocities.
This is a strong result and has not been achieved in existing work.
In Figs. \ref{fig:linear velocity tracking}-\ref{fig:angular velocity tracking}, both the actual and estimated velocities exhibited high-frequency oscillations.
Note that this issue did not arise from our observer design.
Instead, it was due to our assumption of the linear constitutive law \eqref{eq:linear constitutive law}, which necessarily resulted in high-frequency oscillations in the actual robot states.
Despite the oscillations, the estimates showed close tracking of the actual states.
This was confirmed by the $L^\infty$ norms of estimation errors plotted in Fig.~\ref{fig:estimation error}. 
We observed that the estimation errors of all robot states converged to a small neighborhood of zero in 0.5 seconds.
The steady-state error of the velocities observed after 1.5 seconds occurred during high-frequency oscillations of the robot's actual states, which made the estimation process much harder. 
However, this simulation result was consistent with our theoretical result in Theorem~\ref{thm:stability}, which stated that the estimation error would converge to a small region bounded by the actual states.
Fig.~\ref{fig:estimation error} would be used as a baseline to discuss the robustness of our observer.

Next, we conducted a series of independent simulation studies to study the impact of different factors.
For each study, we performed simulations for the three initial conditions in Fig.~\ref{fig:initial condition}.
We would only plot the $L^\infty$ norms of estimation errors.

\begin{figure}[!htb]
    \centering
    \includegraphics[width=0.9\columnwidth]{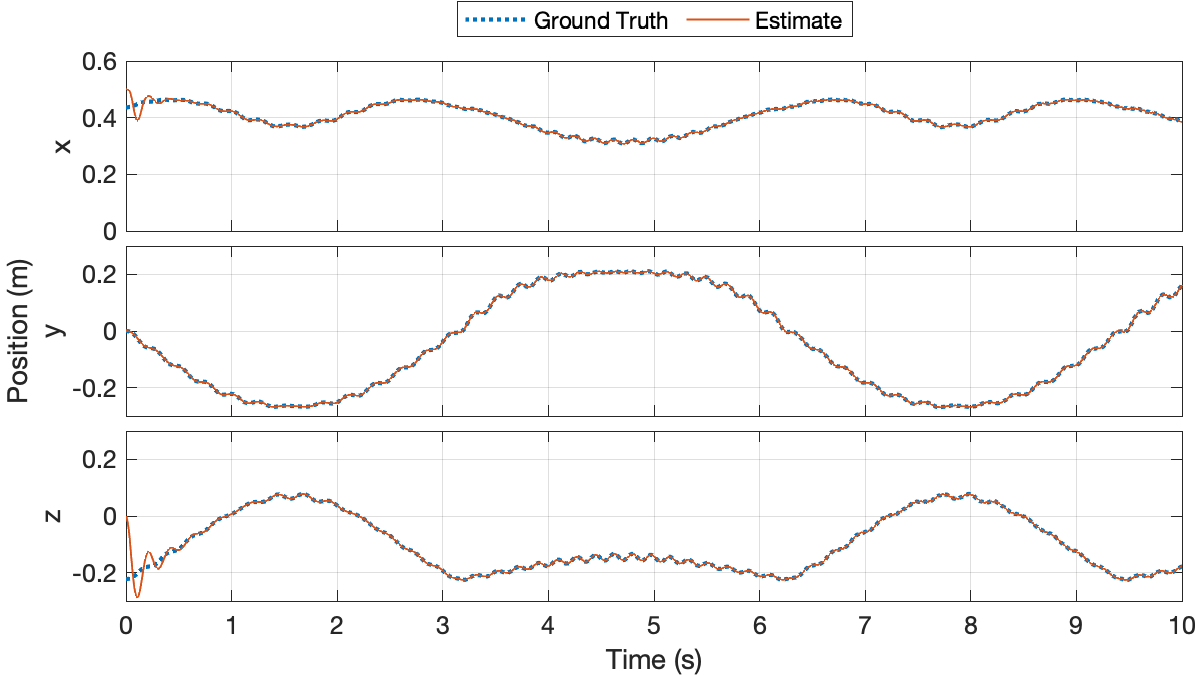}
    \caption{Ground truth and estimates of the tip position starting with initial condition 1.}
    \label{fig:position tracking}
\end{figure}

\begin{figure}[!htb]
    \centering
    \includegraphics[width=0.9\columnwidth]{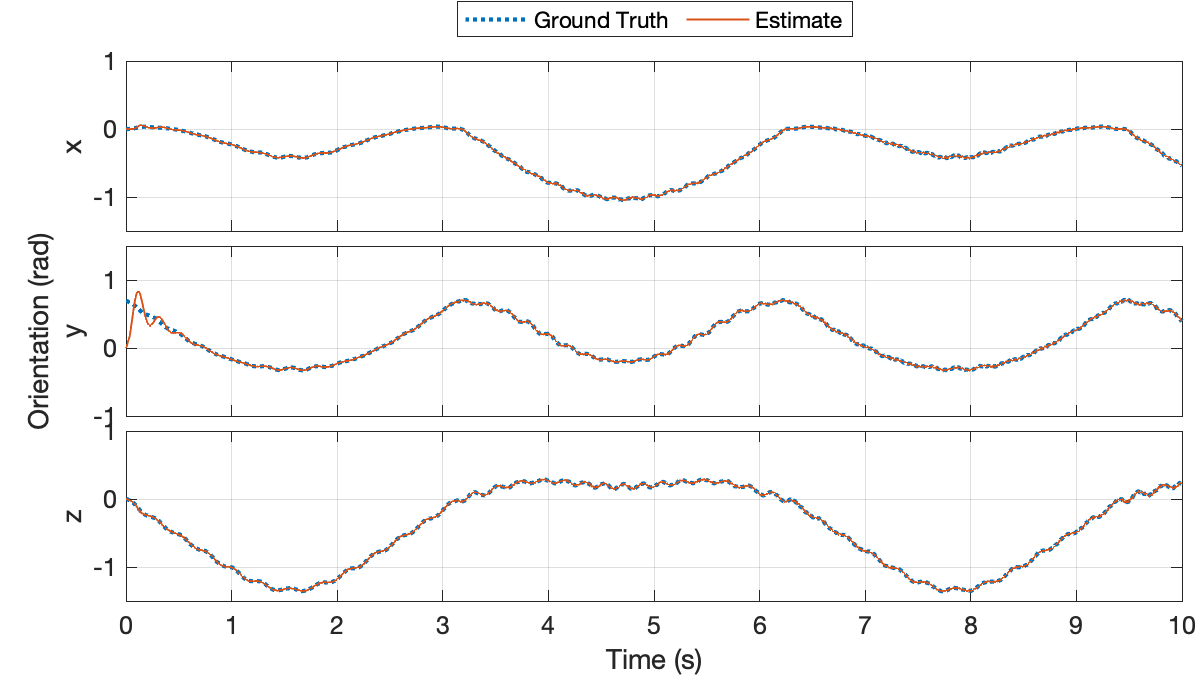}
    \caption{Ground truth and estimates of the tip orientation starting with initial condition 1.}
    \label{fig:rotation tracking}
\end{figure}

\begin{figure}[!htb]
    \centering
    \includegraphics[width=0.9\columnwidth]{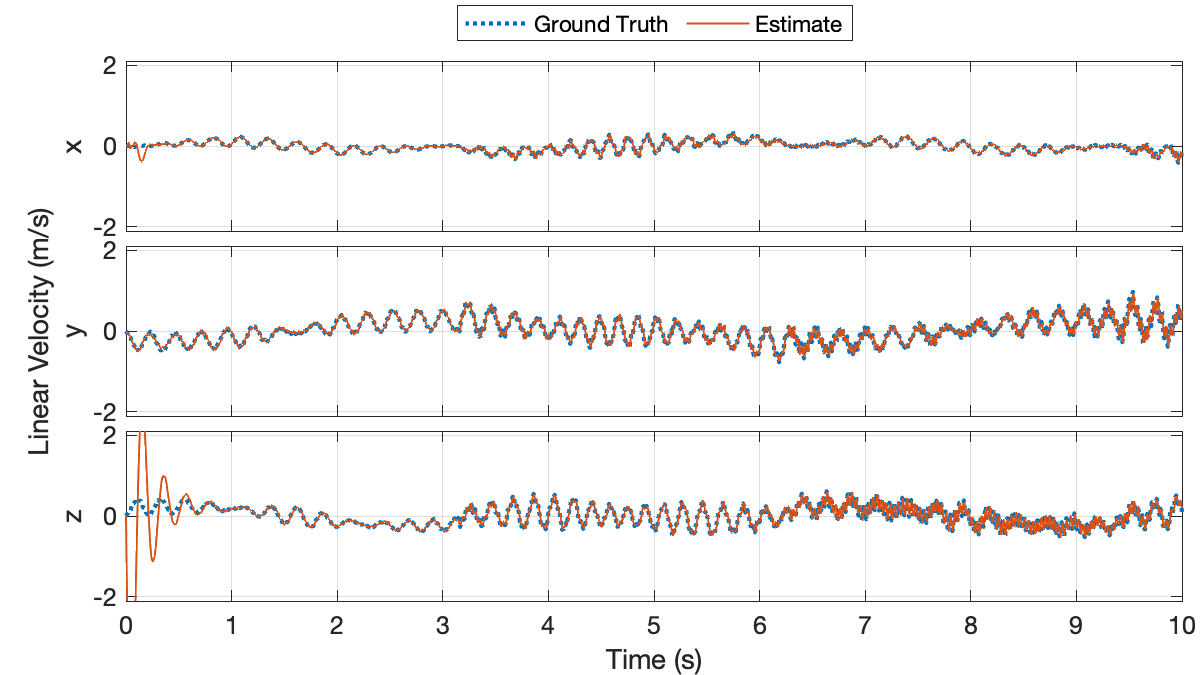}
    \caption{Ground truth and estimates of the tip linear velocity starting with initial condition 1.}
    \label{fig:linear velocity tracking}
\end{figure}

\begin{figure}[t]
    \centering
    \includegraphics[width=0.9\columnwidth]{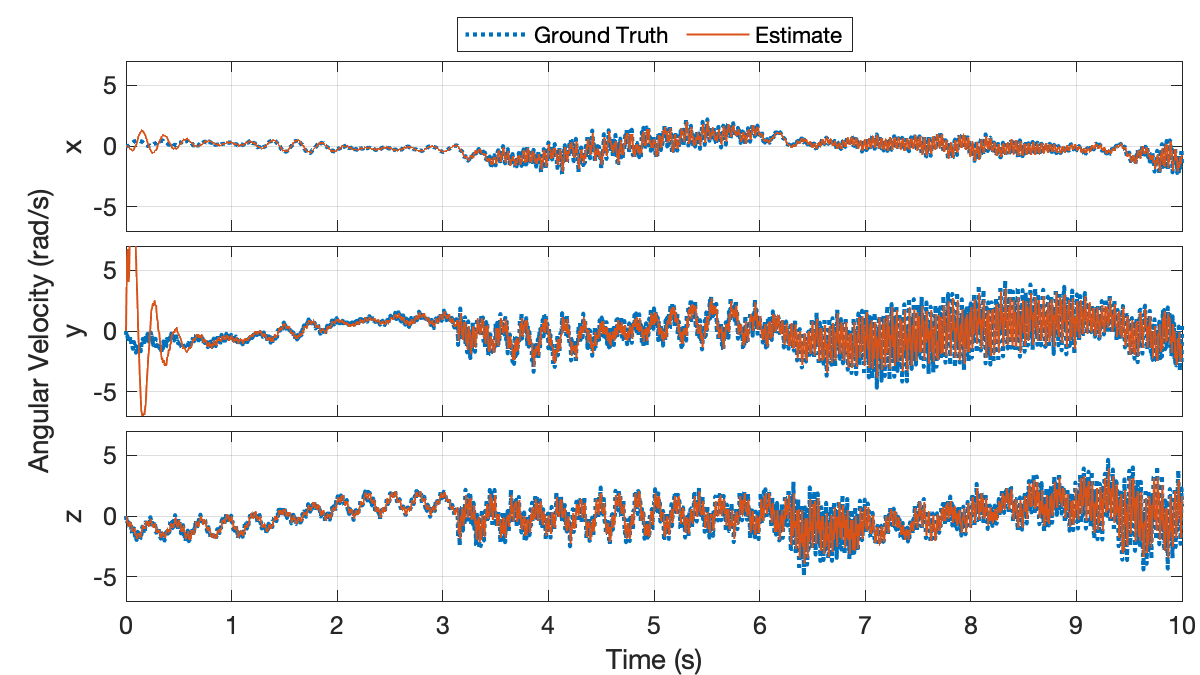}
    \caption{Ground truth and estimates of the tip angular velocity starting with initial condition 1.}
    \label{fig:angular velocity tracking}
\end{figure}

\begin{figure}[!htb]
    \centering
    \includegraphics[width=0.9\columnwidth]{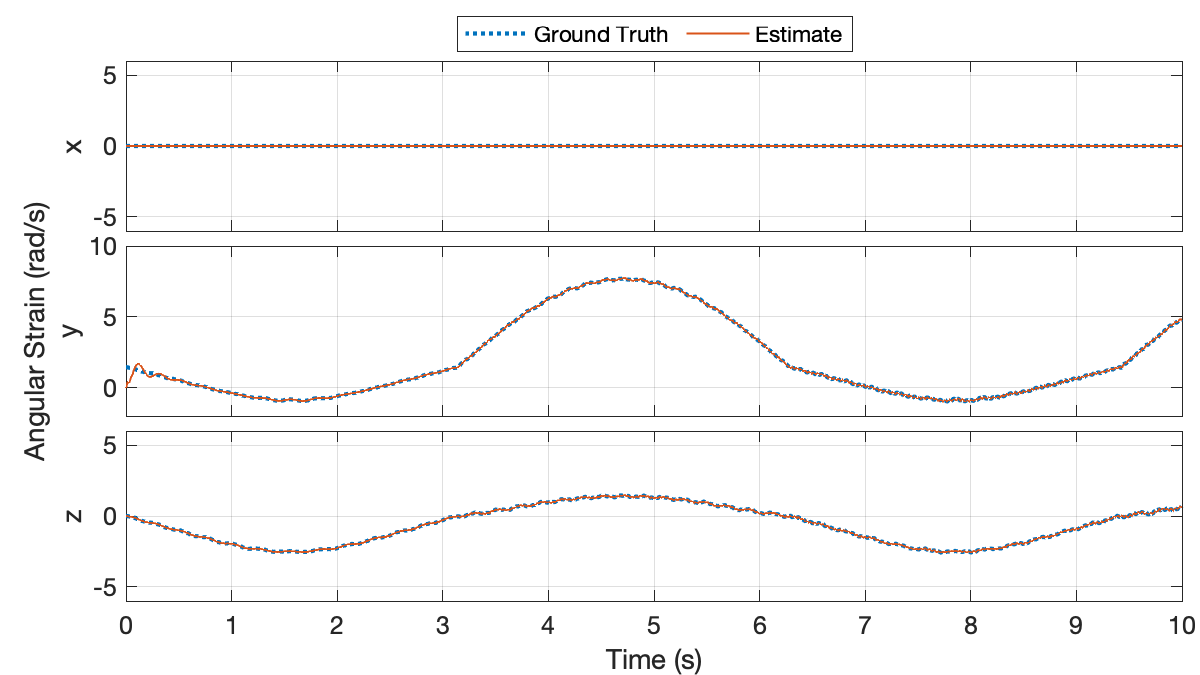}
    \caption{Ground truth and estimates of the angular strain at the midpoint location starting with initial condition 1. A spring backbone has negligible torsion so the $x$-coordinate is zero.}
    \label{fig:angular strain tracking}
\end{figure}


\subsection{Study 1: Impact of Tip Measurement Noise}
This section studies the impact of tip velocity measurement noise on the observer. 
At every $t$, independent noise was drawn from a uniform distribution on the interval $[-20\%,20\%]$ of the maximum magnitude of the tip linear/angular velocity and then added to each component of the tip linear/angular velocity measurements.
Note that this is a significant level of noise.

\textit{Result.} 
The $L^\infty$ norms of estimation errors are shown in Fig.~\ref{fig:estimation error noise}.
We observed that tip measurement noise propagated to all robot state estimates, especially on the velocity estimates.
However, the steady-state estimation errors still remained in a small region of zero.
This suggests that the boundary observer is robust to tip measurement noise.

\subsection{Study 2: Impact of Robot Modeling Errors}
This section studies the impact of robot modeling errors.
Modeling errors can occur in all physical parameters of the robot, such as its radius, density, Young's modulus, and shear modulus.
Since there were too many physical parameters, we directly added perturbations to the parameters $J(s)$ and $K(s)$, assuming that this was the consequence of errors in the physical parameters.
The perturbed parameters were set to
\begin{align*}
    \Bar{J}(s) & =J(s)*(1+0.2*\sin(20s)), \\
    \Bar{K}(s) & =K(s)*(1+0.2*\sin(20s)).
\end{align*}
We used $J,K$ to compute the actual states and $\Bar{J},\Bar{K}$ to compute the estimates.

\textit{Result.} 
The $L^\infty$ norms of estimation errors are shown in Fig.~\ref{fig:estimation error JK}.
We observed that modeling errors on the physical parameters also led to additional steady-state errors in the estimation.
However, these errors still remained close to zero.
This suggests that the boundary observer is robust to robot modeling errors as well.

\begin{figure}[!htb]
    \centering
    \includegraphics[width=0.9\columnwidth]{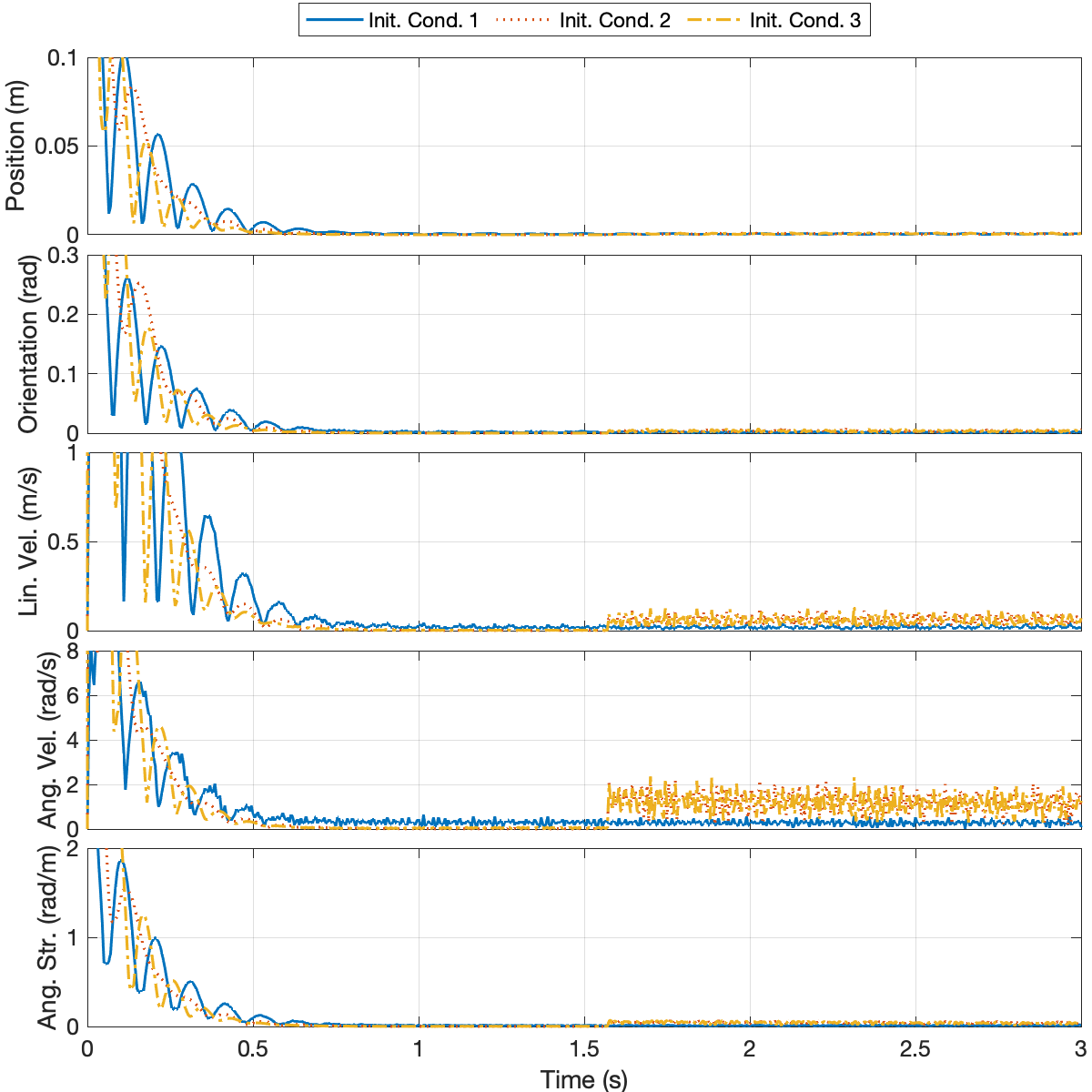}
    \caption{Baseline estimation errors for three initial conditions.}
    \label{fig:estimation error}
\end{figure}

\begin{figure}[!htb]
    \centering
    \includegraphics[width=0.9\columnwidth]{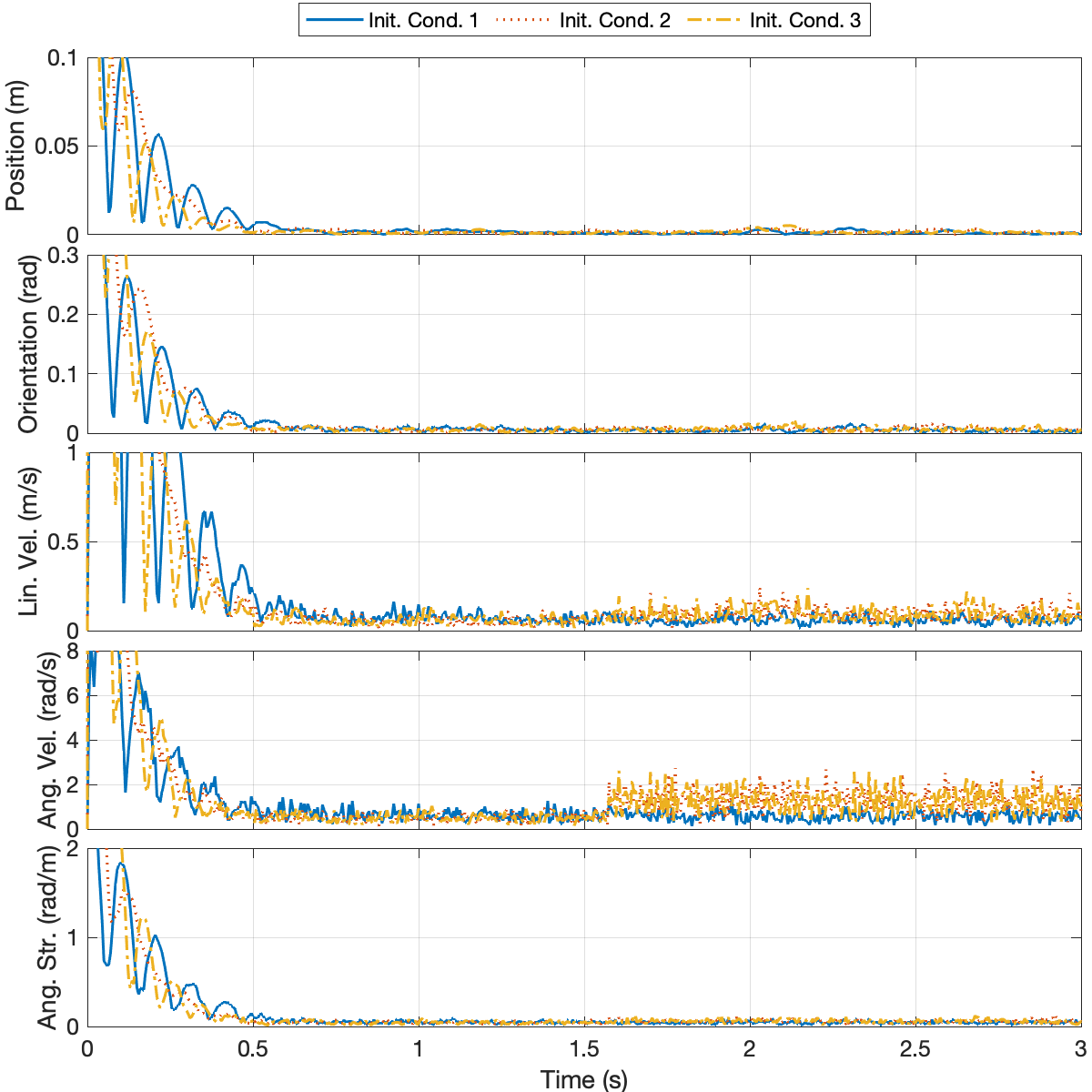}
    \caption{Estimation errors for Study 1 in the presence of tip measurement noise.}
    \label{fig:estimation error noise}
\end{figure}

\begin{figure}[!htb]
    \centering
    \includegraphics[width=0.9\columnwidth]{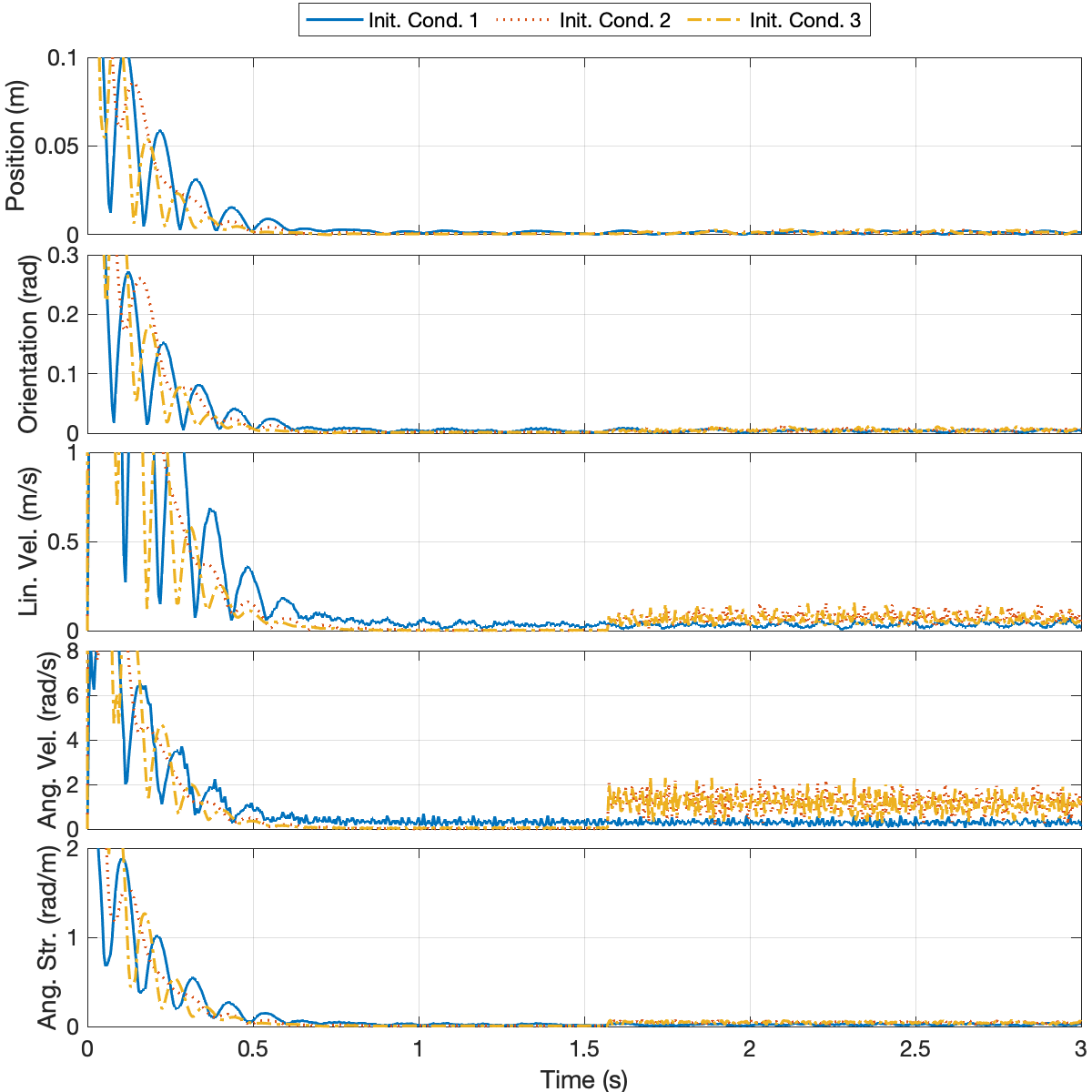}
    \caption{Estimation errors for Study 2 in the presence of modeling errors on the robot parameters $J(s),K(s)$.}
    \label{fig:estimation error JK}
\end{figure}

\begin{figure}[!htb]
    \centering
    \includegraphics[width=0.9\columnwidth]{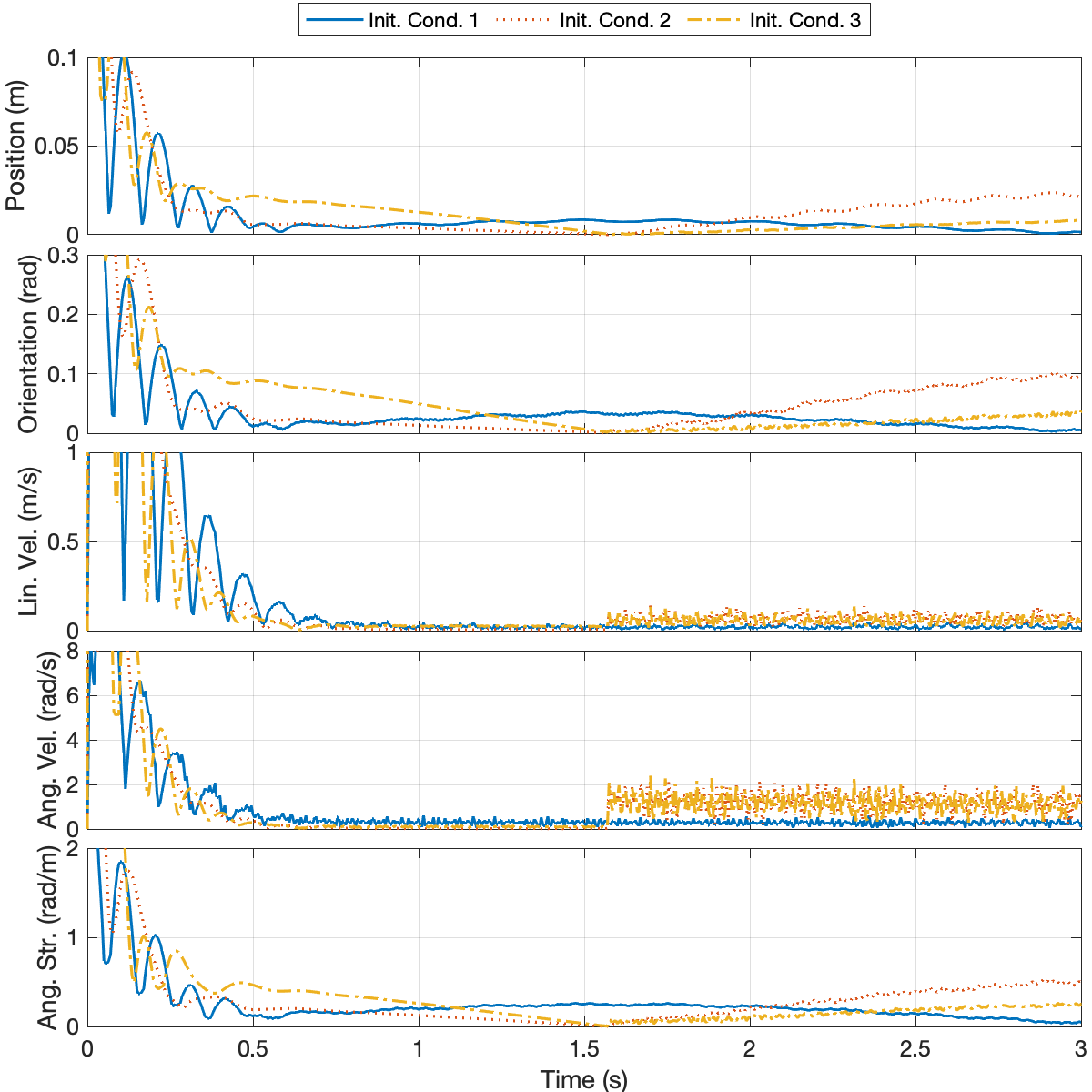}
    \caption{Estimation errors for Study 3 in the presence of modeling errors on the tendons routing parameter $D_i(s)$.}
    \label{fig:estimation error tendon}
\end{figure}

\subsection{Study 3: Impact of actuator modeling errors}
This section studies the impact of actuator modeling errors.
For a tendon-driven robot, actuator modeling errors can occur in the tendons routing parameter, such as its relative position to the backbone.
Hence, we added perturbations to the tendon positions $D_i(s)$.
The perturbed tendon positions were assumed
\begin{align*}
    \bar{D}_i(s) & =D_i(s)*(1+0.1*\sin(20s)).
\end{align*}
We used $D_i$ to compute the actual states and $\Bar{D}_i$ to compute the estimates.

\textit{Result.} 
The $L^\infty$ norms of estimation errors are shown in Fig.~\ref{fig:estimation error tendon}.
We observed that modeling errors on the tendons routing parameter introduced a notable amount of steady-state estimation errors compared to previous studies, especially on the positions, orientations, and angular strains.
This is not surprising because the central idea of our boundary observer is to dissipate energy from the estimation errors, and inaccurate actuator models may inject extra energy into the estimation errors.}




\subsection{Summary}
In summary, although the theoretical stability is only local, numerical studies suggest that the domain of attraction is large.
The observer is more sensitive to modeling errors on the actuator model compared to the tip measurement noise and modeling errors on the robot parameters.
Future work will study the use of additional measurements (like position) to compensate for actuator modeling errors.

\section{Conclusion}
\label{section:conclusion}

In this work, we designed a boundary observer for continuum robotic arms based on Cosserat rod PDEs.
This observer was able to estimate all infinite-dimensional continuum robot states from only tip velocity measurements, could be easily implemented in both the hardware and numerical aspects and was proven to be locally input-to-state stable.
Extensive numerical simulations suggested that the domain of attraction was large and the observer was robust to tip velocity measurement noise and robot modeling errors.
These results suggested the promising role of PDE control theory for soft robots.
Our future work includes theoretical studies on the robustness of the observer, using additional sensing data to improve the robustness, and testing it on physical platforms.

\section*{Appendix}

\subsection{Input-to-State Stability}

\black{Input-to-state stability (ISS) is a concept used to analyze nonlinear control systems with external inputs \cite{sontag1989smooth}. 
The extension to infinite-dimensional control systems, such as those of partial differential equations, was introduced by \cite{dashkovskiy2013input}.

Let $\left(X,\|\cdot\|_X\right)$ and $\left(U,\|\cdot\|_{U}\right)$ be the state space and the input space, endowed with norms $\|\cdot\|_X$ and $\|\cdot\|_{U}$, respectively.
Denote $U_c=PC(\mathbb{R}_+;U)$, the space of piecewise right-continuous functions from $\mathbb{R}_+$ to $U$, equipped with the sup-norm.
Define the following classes of comparison functions:
\begin{align*}
    \mathcal{P}&:=\{\gamma: \mathbb{R}_+ \to \mathbb{R}_+ \mid \gamma\text{ is continuous},~\gamma(0)=0, \\
    &\qquad\text{and }\gamma(r)>0, ~\forall r>0\} \\
    \mathcal{K} &:=\{\gamma\in\mathcal{P} \mid \gamma\text{ is strictly increasing}\} \\
    \mathcal{K}_{\infty} &:=\{\gamma\in\mathcal{K} \mid \gamma\text{ is unbounded}\}  \\
    \mathcal{L} &:=\{\gamma:\mathbb{R}_+\to\mathbb{R}_+ \mid \gamma\text{ is continuous and strictly } \\
    &\qquad\text{decreasing with }\lim_{t\to\infty}\gamma(t)=0\}  \\
    \mathcal{KL} &:=\{\beta:\mathbb{R}_+\times\mathbb{R}_+\to\mathbb{R}_+ \mid \beta(\cdot,t)\in\mathcal{K},~\forall t\geq0, \\
    &\qquad\beta(r,\cdot)\in\mathcal{L},~\forall r>0\}.
\end{align*}

Consider a control system $\Sigma=(X,U_c,\phi)$ where $\phi:\mathbb{R}_+\times X\times U_c\to X$ is a transition map.
Here, $\phi(t,x_0,u(\cdot))$ denotes the state of the system at time $t\in\mathbb{R}_+$, if its initial state was $x_0\in X$ and the input $u(\cdot)\in U_c$ was applied.
Denote $x(t)=\phi(t,x_0,u(\cdot))$.

\begin{definition} 
\label{dfn:(L)ISS}
$\Sigma$ is called \textit{locally input-to-state stable (LISS)}, if $\exists k_x,k_u>0$, $\beta\in\mathcal{KL}$, and $\gamma\in\mathcal{K}$, such that
\begin{equation}\label{eq:(L)ISS}
    \|x(t)\|_X \leq \beta(\|x(0)\|_X, t)+\gamma\Big(\sup_{0\leq\tau\leq t}\|u(\tau)\|_{U}\Big)
\end{equation}
holds for $\forall x(0):\|x(0)\|_X\leq k_x$, $\forall u:\sup_{0\leq\tau\leq t}\|u(\tau)\|_{U}\leq k_u$, and $\forall t\geq0$.
It is called \textit{input-to-state stable (ISS)}, if $ k_x=\infty$ and $ k_u=\infty$.
\end{definition} 

The (L)ISS property can be concluded by showing the existence of a so-called (L)ISS-Lyapunov functional.

\begin{definition}
\label{dfn:(L)ISS-Lyapunov function}
Let $r_x,r_u>0$.
Let $D_x=\{x\in X\mid\|x\|_X\leq r_x\}$ and $D_u=\{u\in U\mid\|u\|_U\leq r_u\}$.
A continuous functional $V:D_x\to\mathbb{R}$ is called an \textit{LISS-Lyapunov functional} for $\Sigma$, if $\exists\alpha_1,\alpha_2\in\mathcal{K}_\infty,\rho\in\mathcal{K}$, and $W\in\mathcal{P}$, such that:
\begin{align*}
    & \alpha_1(\|x\|_X) \leq V(x) \leq \alpha_2(\|x\|_X), \\
    \text{and } & \dot{V}(x) \leq-W(\|x\|_X), \quad \forall \|x\|_X \geq \rho(\|u\|_U), 
\end{align*}
$\forall x\in D_x$, $\forall u\in D_u$, and $\forall t\geq0$, where
\begin{align*}
    \dot{V}(x)=\varlimsup_{\delta t \rightarrow+0} \frac{1}{\delta t}\Big(V\big(\phi(\delta t, x, u(\cdot))\big)-V(x)\Big) .
\end{align*}
If $r_x=\infty$ and $r_u=\infty$, then $V$ is called an \textit{ISS-Lyapunov functional}.
\end{definition}

\begin{theorem}[\cite{dashkovskiy2013input, zheng2021transporting}]
\label{thm:(L)ISS-Lyapunov function}
A control system $\Sigma$ is (L)ISS if it possesses an (L)ISS-Lyapunov functional.
\end{theorem}

\begin{remark}
\label{remark:(L)ISS-Lyapunov function}
In the above theorem, the resulting constants and comparison functions in Definition~\ref{dfn:(L)ISS} depend on those assumed in Definition~\ref{dfn:(L)ISS-Lyapunov function} according to \cite{khalil2002nonlinear}:
\begin{align*}
    k_x & = \alpha_2^{-1}(\alpha_1(r_x)), \\
    k_u & = \rho^{-1}(\min\{k_x,\rho(r_u)\}), \\
    \gamma & =\alpha_1^{-1} \circ \alpha_2 \circ \rho.
\end{align*}
A useful special case is when the comparison functions $\alpha_1,\alpha_2,W$ are all quadratic, and $\rho$ is linear.
In this case, the resulting ISS inequality \eqref{eq:(L)ISS} takes the special form:
\begin{align}
    \|x(t)\|_X \leq b\|x(0)\|_Xe^{-\lambda t}+\kappa\sup_{0\leq\tau\leq t}\|u(\tau)\|_{U},
\end{align}
where $b,\lambda,\kappa>0$ are constants.    
\end{remark}
}

\subsection{Boundary Stabilization of Cosserat Rod PDEs}

The proof of Theorem \ref{thm:stability} relies on some preliminary results from the work \cite{rodriguez2022networks}, which studies boundary stabilization of Cosserat rods.
The main results are summarized here.
Consider the following system with boundary control:
\begin{align}\label{eq:PDE system with boundary dissipation}
\left\{
\begin{aligned}
    & \partial_t\xi=\partial_s\eta+\ad_{\xi}\eta, \\
    & J\partial_t\eta=\partial_s\phi-\ad_{\xi}^T\phi+\ad_{\eta}^TJ\eta, \\
    & \begin{aligned}
    & \phi(\ell,t)=-\Gamma\eta(\ell,t), \\
    & \eta(0,t)=0, \\
    & \phi(s,0)=\phi_0(s), \\
    & \eta(s,0)=\eta_0(s),
    \end{aligned}
\end{aligned}
\right.
\end{align}
where $\Gamma\in\mathbb{R}^{6\times6}$ is the feedback gain.
The system \eqref{eq:PDE system with boundary dissipation} is related to \eqref{eq:PDE system} by (i) setting all the distributed inputs $\{\phi_{\text{loc}},\psi_{\text{loc}},\psi_{\text{glb}}\}$ (including gravity) to be 0, (ii) setting the boundary condition $\eta(0,t)$ to be 0, and (iii) setting the boundary condition $\phi(\ell,t)$ to be $-\Gamma\eta(\ell,t)$, which is the boundary controller in \cite{rodriguez2022networks}.

The total energy of \eqref{eq:PDE system with boundary dissipation}, consisting of kinetic energy and (elastic) potential energy, is defined by
\begin{align*}
    \mathcal{E}(t)=\int_0^\ell\eta^TJ\eta+\phi^TK^{-1}\phi ds.
\end{align*}
If $\Gamma=0$, it is well-known that the total energy is conserved, i.e., $\frac{d}{dt}\mathcal{E}(t)=0$.
If $\Gamma>0$, one can show that 
\begin{align*}
    \frac{d}{dt}\mathcal{E}(t)=-2\eta^T(\ell,t)\Gamma\eta(\ell,t)\leq0,
\end{align*}
which implies the total energy is non-increasing, i.e., \eqref{eq:PDE system with boundary dissipation} is stable.
However, it is not necessarily asymptotically stable.
The authors of \cite{rodriguez2022networks} then proved that \eqref{eq:PDE system with boundary dissipation} is locally exponentially stable by finding a modulated quadratic Lyapunov functional.
In particular, \eqref{eq:PDE system with boundary dissipation} can be rewritten as
\begin{align}\label{eq:compact PDE with boundary dissipation}
\left\{
\begin{aligned}
    & \partial_ty+A\partial_sy+By=F(y) \\
    & \begin{aligned}
    & \phi(\ell,t)=-\Gamma\eta(\ell,t), \\
    & \eta(0,t)=0, \\
    & y(s,0)=y_0(s),
    \end{aligned}
\end{aligned}
\right.
\end{align}
where $A$, $B$, and $F$ are defined in the same way as in \eqref{eq:compact error PDE}.
The following theorems summarize the main results from \cite{rodriguez2022networks}.

\begin{theorem}[Well-Posedness \cite{rodriguez2022networks}]
For any $T>0$, there exists $\delta>0$ such that if $y_0$ satisfies $\|y_0\|_{H^1}\leq\delta$ and the compatibility conditions, then there exists a unique solution $y\in C^0([0,T),H^1)$ to \eqref{eq:compact PDE with boundary dissipation}.
Moreover, if $\|y(\cdot,t)\|_{H^1}\leq\delta$ for all $t\in[0,T)$, then $T=+\infty$.
\end{theorem}

\black{
\begin{theorem}[Local Exponential Stability \cite{rodriguez2022networks}]
\label{thm:boundary stabilization}
If $\Gamma$ is positive definite, then \eqref{eq:compact PDE with boundary dissipation} is locally exponentially stable in $H^1$ and there exists $P\in C^1([0,\ell];\mathbb{R}^{12\times12})$ satisfying
\begin{enumerate}
    \item $P(s)$ is positive definite for all $s\in[0,\ell]$, 
    \item $(PA)(s)$ is symmetric for all $s\in[0,\ell]$,
    \item $Q:=-\frac{d}{ds}(PA)+PB+B^TP$ is positive definite for all $s\in[0,\ell]$,
    \item $(y^TPAy)(\ell,t)-(y^TPAy)(0,t)\geq0$ for all $t$,
\end{enumerate}
such that the following function defines a Lyapunov certificate:
\begin{align*}
    \mathcal{V}=\int_0^\ell y^TPy+\partial_ty^TP\partial_tyds,
\end{align*}
which satisfies
\begin{align*}
    \dot{\mathcal{V}}\leq\int_0^\ell -y^TQy-\partial_ty^TQ\partial_tyds.
\end{align*}
\end{theorem}

The proof of the existence of $P$ and the procedure for constructing $P$ can be found in \cite{rodriguez2022networks}.
The properties of $P$ are crucial for the proof of Theorem~\ref{thm:stability}.
}

\subsection{Proof of Theorem \ref{thm:stability}}

\begin{proof}
By suitably renaming the variables, \eqref{eq:compact error PDE} can be seen as a perturbation of \eqref{eq:compact PDE with boundary dissipation} with the additional nonlinear terms $G$ and $H$.
The main idea of our proof is to extend the original proof (in \cite{rodriguez2022networks}) of Theorem \ref{thm:boundary stabilization} to accommodate the terms $G$ and $H$ and establish local input-to-state stability.
Throughout the proof, all norms are defined only on the $s$ variable, and we also omit the dependence on $t$ for brevity\footnote{For instance, for a function $f(s,t)$, we denote $\|f\|_{L^2}=\|f(\cdot,t)\|_{L^2}$ which is a function of $t$ but we omit $t$.}.

Let $P$ satisfy all the properties in Theorem \ref{thm:boundary stabilization}.
Define Lyapunov functionals
\begin{align*}
    \mathcal{V}_0(t) & =\int_0^\ell \tilde{y}^TP\tilde{y}ds, \\
    \mathcal{V}_1(t) & =\int_0^\ell\partial_t\tilde{y}^TP\partial_t\tilde{y}ds.
\end{align*}
Let $\mathcal{V}=\mathcal{V}_0+\mathcal{V}_1$.
\black{It is easy to verify that $\sqrt{\mathcal{V}}$ defines a norm.
We will show that $\sqrt{\mathcal{V}}$ is equivalent to $\|\tilde{y}\|_{H^1}$} later in \eqref{eq:norm equivalence}.
Differentiating $\mathcal{V}_0$, substituting \eqref{eq:compact error PDE}, and then using integration by parts and the properties of $P$, we have
\begin{align*}
    \dot{\mathcal{V}}_0 & =\int_0^\ell2 \tilde{y}^TP\big[-A\partial_s\tilde{y}-B\tilde{y} \\
    & \quad +F(\tilde{y})+G(y,\tilde{y})+H(d,\tilde{R},\tilde{y})\big]ds \\
    & =-\tilde{y}^TPA\tilde{y}\mid_0^\ell + \int_0^\ell\tilde{y}^T\partial_s(PA)\tilde{y}+\partial_s\tilde{y}^TPA\tilde{y} \\
    & \quad -\tilde{y}^TPA\partial_s\tilde{y}-\tilde{y}^TPB\tilde{y}-\tilde{y}^TB^TP\tilde{y} \\
    & \quad +\tilde{y}^TP\big[F(\tilde{y})+G(y,\tilde{y})+H(d,\tilde{R},\tilde{y})\big]ds \\
    & \leq\int_0^\ell-\tilde{y}^TQ\tilde{y}+\tilde{y}^TP\big[F(\tilde{y})+G(y,\tilde{y})+H(d,\tilde{R},\tilde{y})\big]ds,
\end{align*}
where the term $\tilde{y}^TPA\tilde{y}\mid_0^\ell\geq0$ by the forth property of $P$ and $Q:=-\partial_s(PA)+PB+B^TP$ is positive definite by the third property of $P$ in Theorem \ref{thm:boundary stabilization}.
By a similar derivation, 
\begin{align*}
    \dot{\mathcal{V}}_1 & \leq\int_0^\ell-\partial_t\tilde{y}^TQ\partial_t\tilde{y} \\
    & \quad +\partial_t\tilde{y}^TP\partial_t\big[F(\tilde{y})+G(y,\tilde{y})+H(d,\tilde{R},\tilde{y})\big]ds.
\end{align*}
Thus, there exists a constant $C_0>0$ such that
\begin{align}\label{eq:Q term}
\begin{split}
    \dot{\mathcal{V}} & \leq -C_0\mathcal{V}+\sum_{i=0}^1\int_0^\ell\partial_t^i\tilde{y}^TP\partial_t^i\big[F(\tilde{y}) \\
    & \quad +G(y,\tilde{y})+H(d,\tilde{R},\tilde{y})\big]ds.
\end{split}
\end{align}

We need to derive some estimates for the nonlinear terms $F$, $G$, and $H$.
From now on, we assume there exist constants ${r_0}$, ${r_1}$, ${r_2}>0$ such that $\|\tilde{y}\|_{H^1}<{r_0}$, $\|y\|_{H^1}<{r_1}$, and $\|d\|_{H}<{r_2}$ for all $t$.
By the Sobolev inequality \cite{evans1998partial}, we correspondingly have $\|\tilde{y}\|_{L^\infty}<c_0{r_0}$, $\|y\|_{L^\infty}<c_1{r_1}$, and $\|d\|_{L^\infty}<c_2{r_2}$ for some constants $c_0,c_1,c_2>0$.

By definition, $F(y)$ is a matrix-valued quadratic function consisting of only square terms.
By H\"older's inequality \cite{evans1998partial}, there exist constants $c_3,c_4>0$ such that 
\begin{align*}
    & \int_0^\ell\tilde{y}^TPF(\tilde{y})ds\leq c_3\|\tilde{y}\|_{L^\infty}\|\tilde{y}\|_{L^2}^2\leq c_0c_3{r_0}\|\tilde{y}\|_{L^2}^2, \\
    & \int_0^\ell\partial_t\tilde{y}^TP\partial_tF(\tilde{y})ds\leq c_4\|\tilde{y}\|_{L^\infty}\|\partial_t\tilde{y}\|_{L^2}^2\leq c_0c_4{r_0}\|\partial_t\tilde{y}\|_{L^2}^2.
\end{align*}
The above estimates imply that there exists a constant $C_1>0$ such that 
\begin{align}\label{eq:F term}
    \sum_{i=0}^1\int_0^\ell\partial_t^i\tilde{y}^TP\partial_t^iF(\tilde{y})ds\leq C_1{r_0}\mathcal{V}.
\end{align}

By definition, $G(y,\tilde{y})$ is a matrix-valued bilinear function. 
By H\"older's inequality, there exist constants $c_5,c_6,c_7>0$ such that
\begin{align*}
    \int_0^\ell & \tilde{y}^TPG(y,\tilde{y})ds \\
    & \leq c_5\|\tilde{y}\|_{L^\infty}\|\tilde{y}\|_{L^2}\|y\|_{L^2} \leq c_1c_5{r_0}\|\tilde{y}\|_{L^2}\|y\|_{L^2}, \\
    \int_0^\ell & \partial_t\tilde{y}^TP\partial_tG(y,\tilde{y})ds \\
    & \leq c_6\|\tilde{y}\|_{L^\infty}\|\partial_t\tilde{y}\|_{L^2}\|\partial_ty\|_{L^2} + c_7\|y\|_{L^\infty}\|\partial_t\tilde{y}\|_{L^2}^2 \\
    & \leq c_0c_6{r_0}\|\partial_t\tilde{y}\|_{L^2}\|\partial_ty\|_{L^2} + c_1c_7{r_1}\|\partial_t\tilde{y}\|_{L^2}^2.
\end{align*}
The above estimates, together with the fact that $\sum_{i=0}^1\|\partial_t^iy\|_{L^2}$ is locally equivalent to $\|y\|_{H^1}$ (\cite{rodriguez2022networks}, p.p. 23), imply that there exist constants $C_2,C_3>0$ such that
\begin{align}\label{eq:G term}
\begin{split}
    \sum_{i=0}^1\int_0^\ell & \partial_t^i\tilde{y}^TP\partial_t^iG(y,\tilde{y})ds \\
    & \leq C_2{r_0}\sqrt{\mathcal{V}}\|y\|_{H^1} + C_3{r_1}\mathcal{V}.
\end{split}
\end{align}

According to the definition of $\Bar{H}$ in \eqref{eq:complete error PDE}, we can denote $H(d,\tilde{R},\tilde{y})=H_1(d,\tilde{y})+H_2(d,\tilde{R})$ where $H_1$ is the portion involving $\ad_{K^{-1}\tilde{\phi}}\phi_{\text{loc}}$ and $H_2$ is the portion involving $\black{\T_{\tilde{R}}^T\psi_{\text{glb}}}$.
By definition, $H_1(d,\tilde{y})$ and $H_2(d,\tilde{R})$ are both matrix-valued bilinear functions.
$H_1(d,\tilde{y})$ can be easily tackled in a similar way as $G(y,\tilde{y})$.
In particular, we can deduce that there exist constants $C_4,C_5>0$ such that
\begin{align}\label{eq:H1 term}
\begin{split}
    \sum_{i=0}^1\int_0^\ell & \partial_t^i\tilde{y}^TP\partial_t^iH_1(d,\tilde{y})ds \\
    & \leq C_4{r_0}\sqrt{\mathcal{V}}\|d\|_{H} + C_5{r_2}\mathcal{V}.
\end{split}
\end{align}
Now we focus on $H_2(d,\tilde{R})$.
According to \eqref{eq:kinematics 2}, $\partial_sR=Ru\cross$ where $u$ is the angular component of the strain $\xi$ and therefore part of $y$.
Since $\tilde{R}(0,t)=0$ for all $t$, by the fundamental theorem of calculus, there exists a constant $c_8>0$ depending on $\ell$ such that
\begin{align*}
    \|\tilde{R}\|_{L^\infty}& \leq c_8\|\partial_s\tilde{R}\|_{L^\infty} \\
    & =c_8\|\hat{R}\hat{u}\cross-\hat{R}u\cross+\hat{R}u\cross-Ru\cross\|_{L^\infty} \\
    & \leq c_8\|\hat{R}\tilde{u}\cross\|_{L^\infty}+c_8\|\tilde{R}\|_{L^\infty}\|u\|_{L^\infty} \\
    & \leq c_9\|\tilde{u}\|_{L^\infty}+c_1c_8{r_1}\|\tilde{R}\|_{L^\infty},
\end{align*}
for some constant $c_9>0$, where we have used the fact that a rotation matrix $\hat{R}$ has a bounded matrix norm.
Thus, by choosing ${r_1}$ to be sufficiently small, we have 
\begin{align*}
    \|\tilde{R}\|_{L^\infty}\leq\frac{c_9}{1-c_1c_8{r_1}}\|\tilde{y}\|_{L^\infty}=:c_{10}\|\tilde{y}\|_{L^\infty},
\end{align*}
for some constant $c_{10}>0$.
By a similar argument as $\partial_s\tilde{R}$ and the fact that $\partial_tR=Rw\cross$ where $w$ is the angular component of the velocity $\eta$ and therefore part of $y$, there exists a constant $c_{11}>0$ such that
\begin{align*}
    \|\partial_t\tilde{R}\|_{L^\infty} & \leq c_{11}\|\tilde{w}\|_{L^\infty}+c_1{r_1}\|\tilde{R}\|_{L^\infty} \\
    & \leq c_{11}\|\tilde{y}\|_{L^\infty}+c_1c_{10}{r_1}\|\tilde{y}\|_{L^\infty}=:c_{12}\|\tilde{y}\|_{L^\infty},
\end{align*}
for some constant $c_{12}>0$.
Now, by H\"older's inequality, there exist constants $c_{13},c_{14}>0$ such that 
\begin{align*}
    \int_0^\ell & \tilde{y}^TPH_2(d,\tilde{R})ds \\
    & \leq c_{13}\|\tilde{y}\|_{L^2}\|\tilde{R}\|_{L^\infty}\|\psi_{\text{glb}}\|_{L^2} \leq c_0c_{10}c_{13}{r_0}\|\tilde{y}\|_{L^2}\|\psi_{\text{glb}}\|_{L^2}, \\
    \int_0^\ell & \partial_t\tilde{y}^TP\partial_tH_2(d,\tilde{R})ds \\
    & \leq \int_0^\ell\partial_t\tilde{y}^TP\big[(\partial_t\tilde{R}^T)\psi_{\text{glb}}+\tilde{R}^T\partial_t\psi_{\text{glb}}\big]ds \\
    & \leq c_{14}\|\partial_t\tilde{y}\|_{L^2}(\|\partial_t\tilde{R}\|_{L^\infty}\|\psi_{\text{glb}}\|_{L^2}+\|\tilde{R}\|_{L^\infty}\|\partial_t\psi_{\text{glb}}\|_{L^2}) \\
    & \leq c_0c_{14}{r_0}\|\partial_t\tilde{y}\|_{L^2}(c_{12}\|\psi_{\text{glb}}\|_{L^2}+c_{10}\|\partial_t\psi_{\text{glb}}\|_{L^2}).
\end{align*}
The above estimates imply that there exists a constant $C_6>0$ such that
\begin{align}\label{eq:H2 term}
\begin{split}
    \sum_{i=0}^1\int_0^\ell \partial_t^i\tilde{y}^TP\partial_t^iH_2(d,\tilde{R})ds \leq C_6{r_0}\sqrt{\mathcal{V}}\|d\|_{H}.
\end{split}    
\end{align}

We also need to show that \black{$\sqrt{\mathcal{V}}$ is equivalent to $\|\tilde{y}\|_{H^1}$}.
Since the system \eqref{eq:compact error PDE} is of first order in space and time, it yields a relationship between $\partial_t\tilde{y}$ and $\partial_s\tilde{y}$.
In particular, using
\begin{align*}
    \partial_t\tilde{y} & =-A\partial_s\tilde{y}-B\tilde{y}+F(\tilde{y})+G(y,\tilde{y})+H(d,\tilde{R},\tilde{y}), \\
    \partial_s\tilde{y} & =A^{-1}\big(-\partial_t\tilde{y}-B\tilde{y}+F(\tilde{y})+G(y,\tilde{y})+H(d,\tilde{R},\tilde{y})\big),
\end{align*}
we can deduce that there exists a constant $\epsilon$, depending on $A$, $B$, ${r_0}$, ${r_1}$, ${r_2}$, such that
\begin{align}\label{eq:norm equivalence}
    \frac{1}{\epsilon}\sum_{i=0}^1|\partial_s^i\tilde{y}|^2\leq\sum_{i=0}^1|\partial_t^i\tilde{y}|^2\leq\epsilon\sum_{i=0}^1|\partial_s^i\tilde{y}|^2,
\end{align}
which implies that \black{$\sqrt{\mathcal{V}}$ and $\|\tilde{y}\|_{H^1}$ are equivalent.
Thus, any stability results for $\sqrt{\mathcal{V}}$ implies equivalent results for $\|\tilde{y}\|_{H^1}$ with the adjustment of multiplicative positive constants.}

Now, using \eqref{eq:Q term}-\eqref{eq:H2 term}, we have
\begin{align*}
    \dot{\mathcal{V}} & \leq -(C_0-C_1{r_0}-C_3{r_1}-C_5{r_2})\mathcal{V} \\
    & \quad +{r_0}\sqrt{\mathcal{V}}\big(C_2\|y\|_{H^1}+(C_4+C_6)\|d\|_{H}\big) \\
    & =:-C\mathcal{V}+\sqrt{\mathcal{V}}\mathcal{U},
\end{align*}
where $C=C_0-C_1{r_0}-C_3{r_1}-C_5{r_2}$ and $\mathcal{U}={r_0}\big(C_2\|y\|_{H^1}+(C_4+C_6)\|d\|_{H}\big)$.
We can let ${r_0},{r_1},{r_2}$ be sufficiently small such that $C>0$.
Let $\theta\in(0,1)$ be a constant.
\black{We have
\begin{align*}
    \dot{\mathcal{V}} & \leq -C(1-\theta)\mathcal{V}-C\theta\mathcal{V}+\sqrt{\mathcal{V}}\mathcal{U} \\
    & = -C(1-\theta)\mathcal{V}-\sqrt{\mathcal{V}}(C\theta\sqrt{\mathcal{V}}-\mathcal{U}).
\end{align*}
Whenever $\sqrt{\mathcal{V}}\geq\frac{1}{C\theta}\mathcal{U}$, we have
\begin{align*}
    \dot{\mathcal{V}} \leq -C(1-\theta)\mathcal{V}.
\end{align*}
According to Definition~\ref{dfn:(L)ISS-Lyapunov function}, $\mathcal{V}$ is an LISS-Lyapunov functional.
Invoking Theorem~\ref{thm:(L)ISS-Lyapunov function}, Remark~\ref{remark:(L)ISS-Lyapunov function}, and the equivalence of $\sqrt{\mathcal{V}}$ and $\|\tilde{y}\|_{H^1}$, we obtain \eqref{eq:ISS}.}
\end{proof}

\bibliographystyle{IEEEtran}
\bibliography{References}

\begin{thebibliography}{10}
\providecommand{\url}[1]{#1}
\csname url@samestyle\endcsname
\providecommand{\newblock}{\relax}
\providecommand{\bibinfo}[2]{#2}
\providecommand{\BIBentrySTDinterwordspacing}{\spaceskip=0pt\relax}
\providecommand{\BIBentryALTinterwordstretchfactor}{4}
\providecommand{\BIBentryALTinterwordspacing}{\spaceskip=\fontdimen2\font plus
\BIBentryALTinterwordstretchfactor\fontdimen3\font minus
  \fontdimen4\font\relax}
\providecommand{\BIBforeignlanguage}[2]{{%
\expandafter\ifx\csname l@#1\endcsname\relax
\typeout{** WARNING: IEEEtran.bst: No hyphenation pattern has been}%
\typeout{** loaded for the language `#1'. Using the pattern for}%
\typeout{** the default language instead.}%
\else
\language=\csname l@#1\endcsname
\fi
#2}}
\providecommand{\BIBdecl}{\relax}
\BIBdecl

\bibitem{laschi2016soft}
C.~Laschi, B.~Mazzolai, and M.~Cianchetti, ``Soft robotics: Technologies and
  systems pushing the boundaries of robot abilities,'' \emph{Science robotics},
  vol.~1, no.~1, p. eaah3690, 2016.

\bibitem{castano2019model}
M.~L. Casta{\~n}o and X.~Tan, ``Model predictive control-based path-following
  for tail-actuated robotic fish,'' \emph{Journal of Dynamic Systems,
  Measurement, and Control}, vol. 141, no.~7, 2019.

\bibitem{sepulchre2023control}
R.~Sepulchre, ``Control challenges for soft robotics [about this issue],''
  \emph{IEEE Control Systems Magazine}, vol.~43, no.~3, pp. 5--7, 2023.

\bibitem{della2023model}
C.~Della~Santina, C.~Duriez, and D.~Rus, ``Model-based control of soft robots:
  A survey of the state of the art and open challenges,'' \emph{IEEE Control
  Systems Magazine}, vol.~43, no.~3, pp. 30--65, 2023.

\bibitem{till2019real}
J.~Till, V.~Aloi, and C.~Rucker, ``Real-time dynamics of soft and continuum
  robots based on cosserat rod models,'' \emph{The International Journal of
  Robotics Research}, vol.~38, no.~6, pp. 723--746, 2019.

\bibitem{mathew2022sorosim}
A.~T. Mathew, I.~M.~B. Hmida, C.~Armanini, F.~Boyer, and F.~Renda, ``Sorosim: A
  matlab toolbox for hybrid rigid-soft robots based on the geometric
  variable-strain approach,'' \emph{IEEE Robotics \& Automation Magazine},
  vol.~30, no.~3, pp. 106--122, 2023.

\bibitem{song2015electromagnetic}
S.~Song, Z.~Li, H.~Yu, and H.~Ren, ``Electromagnetic positioning for tip
  tracking and shape sensing of flexible robots,'' \emph{IEEE Sensors Journal},
  vol.~15, no.~8, pp. 4565--4575, 2015.

\bibitem{bezawada2022shape}
H.~Bezawada, C.~Woods, and V.~Vikas, ``Shape estimation of soft manipulators
  using piecewise continuous pythagorean-hodograph curves,'' in \emph{American
  Control Conference}, 2022, pp. 2905--2910.

\bibitem{anderson2017continuum}
P.~L. Anderson, A.~W. Mahoney, and R.~J. Webster, ``Continuum reconfigurable
  parallel robots for surgery: Shape sensing and state estimation with
  uncertainty,'' \emph{IEEE Robotics and Automation Letters}, vol.~2, no.~3,
  pp. 1617--1624, 2017.

\bibitem{lilge2022continuum}
S.~Lilge, T.~D. Barfoot, and J.~Burgner-Kahrs, ``Continuum robot state
  estimation using gaussian process regression on se(3),'' \emph{The
  International Journal of Robotics Research}, vol.~41, no. 13-14, pp.
  1099--1120, 2022.

\bibitem{armanini2023soft}
C.~Armanini, F.~Boyer, A.~T. Mathew, C.~Duriez, and F.~Renda, ``Soft robots
  modeling: A structured overview,'' \emph{IEEE Transactions on Robotics},
  2023.

\bibitem{chirikjian1994hyper}
G.~S. Chirikjian, ``Hyper-redundant manipulator dynamics: A continuum
  approximation,'' \emph{Advanced Robotics}, vol.~9, no.~3, pp. 217--243, 1994.

\bibitem{della2020model}
C.~Della~Santina, R.~K. Katzschmann, A.~Bicchi, and D.~Rus, ``Model-based
  dynamic feedback control of a planar soft robot: trajectory tracking and
  interaction with the environment,'' \emph{The International Journal of
  Robotics Research}, vol.~39, no.~4, pp. 490--513, 2020.

\bibitem{simo1988dynamics}
J.~C. Simo and L.~Vu-Quoc, ``On the dynamics in space of rods undergoing large
  motions—a geometrically exact approach,'' \emph{Computer methods in applied
  mechanics and engineering}, vol.~66, no.~2, pp. 125--161, 1988.

\bibitem{macchelli2007port}
A.~Macchelli, C.~Melchiorri, and S.~Stramigioli, ``Port-based modeling of a
  flexible link,'' \emph{IEEE transactions on robotics}, vol.~23, no.~4, pp.
  650--660, 2007.

\bibitem{rucker2011statics}
D.~C. Rucker and R.~J. Webster~III, ``Statics and dynamics of continuum robots
  with general tendon routing and external loading,'' \emph{IEEE Transactions
  on Robotics}, vol.~27, no.~6, pp. 1033--1044, 2011.

\bibitem{renda2014dynamic}
F.~Renda, M.~Giorelli, M.~Calisti, M.~Cianchetti, and C.~Laschi, ``Dynamic
  model of a multibending soft robot arm driven by cables,'' \emph{IEEE
  Transactions on Robotics}, vol.~30, no.~5, pp. 1109--1122, 2014.

\bibitem{grazioso2019geometrically}
S.~Grazioso, G.~Di~Gironimo, and B.~Siciliano, ``A geometrically exact model
  for soft continuum robots: The finite element deformation space
  formulation,'' \emph{Soft robotics}, vol.~6, no.~6, pp. 790--811, 2019.

\bibitem{boyer2020dynamics}
F.~Boyer, V.~Lebastard, F.~Candelier, and F.~Renda, ``Dynamics of continuum and
  soft robots: A strain parameterization based approach,'' \emph{IEEE
  Transactions on Robotics}, vol.~37, no.~3, pp. 847--863, 2020.

\bibitem{mattioni2020modelling}
A.~Mattioni, Y.~Wu, H.~Ramirez, Y.~Le~Gorrec, and A.~Macchelli, ``Modelling and
  control of an ipmc actuated flexible structure: A lumped port hamiltonian
  approach,'' \emph{Control Engineering Practice}, vol. 101, p. 104498, 2020.

\bibitem{loo2019h}
J.~Y. Loo, C.~P. Tan, and S.~G. Nurzaman, ``H-infinity based extended kalman
  filter for state estimation in highly non-linear soft robotic system,'' in
  \emph{American Control Conference}, 2019, pp. 5154--5160.

\bibitem{stewart2022state}
K.~Stewart, Z.~Qiao, and W.~Zhang, ``State estimation and control with a robust
  extended kalman filter for a fabric soft robot,'' \emph{IFAC-PapersOnLine},
  vol.~55, no.~37, pp. 25--30, 2022.

\bibitem{zheng2022pde}
T.~Zheng and H.~Lin, ``Pde-based dynamic control and estimation of soft robotic
  arms,'' in \emph{2022 IEEE 61st Conference on Decision and Control
  (CDC)}.\hskip 1em plus 0.5em minus 0.4em\relax IEEE, 2022, pp. 2702--2707.

\bibitem{gazzola2018ForwardInverseProblems}
M.~Gazzola, L.~H. Dudte, A.~G. McCormick, and L.~Mahadevan, ``Forward and
  inverse problems in the mechanics of soft filaments,'' \emph{Royal Society
  Open Science}, vol.~5, no.~6, p. 171628, Jun. 2018.

\bibitem{renda2018discrete}
F.~Renda, F.~Boyer, J.~Dias, and L.~Seneviratne, ``Discrete cosserat approach
  for multisection soft manipulator dynamics,'' \emph{IEEE Transactions on
  Robotics}, vol.~34, no.~6, pp. 1518--1533, 2018.

\bibitem{boyer2022statics}
F.~Boyer, V.~Lebastard, F.~Candelier, F.~Renda, and M.~Alamir, ``Statics and
  dynamics of continuum robots based on cosserat rods and optimal control
  theories,'' \emph{IEEE Transactions on Robotics}, 2022.

\bibitem{murray2017mathematical}
R.~M. Murray, Z.~Li, and S.~S. Sastry, \emph{A mathematical introduction to
  robotic manipulation}.\hskip 1em plus 0.5em minus 0.4em\relax CRC press,
  2017.

\bibitem{vazquez2011backstepping}
R.~Vazquez, M.~Krstic, and J.-M. Coron, ``Backstepping boundary stabilization
  and state estimation of a {$2\times 2$} linear hyperbolic system,'' in
  \emph{2011 50th IEEE conference on decision and control and european control
  conference}.\hskip 1em plus 0.5em minus 0.4em\relax IEEE, 2011, pp.
  4937--4942.

\bibitem{castillo2013boundary}
F.~Castillo, E.~Witrant, C.~Prieur, and L.~Dugard, ``Boundary observers for
  linear and quasi-linear hyperbolic systems with application to flow
  control,'' \emph{Automatica}, vol.~49, no.~11, pp. 3180--3188, 2013.

\bibitem{zheng2022task}
T.~Zheng, Q.~Han, and H.~Lin, ``Task space tracking of soft manipulators:
  Inner-outer loop control based on cosserat-rod models,'' in \emph{2023
  American Control Conference (ACC)}, 2023.

\bibitem{tummers2023cosserat}
M.~Tummers, V.~Lebastard, F.~Boyer, J.~Troccaz, B.~Rosa, and M.~T. Chikhaoui,
  ``Cosserat rod modeling of continuum robots from newtonian and lagrangian
  perspectives,'' \emph{IEEE Transactions on Robotics}, 2023.

\bibitem{bastin2016stability}
G.~Bastin and J.-M. Coron, \emph{Stability and boundary stabilization of 1-d
  hyperbolic systems}.\hskip 1em plus 0.5em minus 0.4em\relax Springer, 2016,
  vol.~88.

\bibitem{li2001semi}
T.-T. Li and Y.~Jin, ``Semi-global c1 solution to the mixed initial-boundary
  value problem for quasilinear hyperbolic systems,'' \emph{Chinese Annals of
  Mathematics}, vol.~22, no.~03, pp. 325--336, 2001.

\bibitem{rodriguez2022networks}
C.~Rodriguez, ``Networks of geometrically exact beams: Well-posedness and
  stabilization,'' \emph{Mathematical Control and Related Fields}, vol.~12,
  no.~1, pp. 49--80, 2022.

\bibitem{renda2020geometric}
F.~Renda, C.~Armanini, V.~Lebastard, F.~Candelier, and F.~Boyer, ``A geometric
  variable-strain approach for static modeling of soft manipulators with tendon
  and fluidic actuation,'' \emph{IEEE Robotics and Automation Letters}, vol.~5,
  no.~3, pp. 4006--4013, 2020.

\bibitem{sontag1989smooth}
E.~D. Sontag, ``Smooth stabilization implies coprime factorization,''
  \emph{IEEE transactions on automatic control}, vol.~34, no.~4, pp. 435--443,
  1989.

\bibitem{dashkovskiy2013input}
S.~Dashkovskiy and A.~Mironchenko, ``Input-to-state stability of
  infinite-dimensional control systems,'' \emph{Mathematics of Control,
  Signals, and Systems}, vol.~25, no.~1, pp. 1--35, 2013.

\bibitem{zheng2021transporting}
T.~Zheng, Q.~Han, and H.~Lin, ``Transporting robotic swarms via mean-field
  feedback control,'' \emph{IEEE Transactions on Automatic Control}, vol.~67,
  no.~8, pp. 4170--4177, 2021.

\bibitem{khalil2002nonlinear}
H.~K. Khalil and J.~W. Grizzle, \emph{Nonlinear systems}.\hskip 1em plus 0.5em
  minus 0.4em\relax Prentice hall Upper Saddle River, NJ, 2002, vol.~3.

\bibitem{evans1998partial}
L.~C. Evans, \emph{{Partial differential equations}}, ser. Graduate Studies in
  Mathematics.\hskip 1em plus 0.5em minus 0.4em\relax Providence, RI: American
  Mathematical Society, 1998.

\end{thebibliography}

\end{document}